\RequirePackage{fix-cm}
\documentclass[smallcondensed]{svjour3}     
\smartqed  
\usepackage{graphicx}
\usepackage{epstopdf}
\usepackage{amsmath}
\usepackage{array}
\usepackage{subfigure}
\usepackage{tabularx}
\usepackage{dsfont}
\usepackage{amssymb}
\usepackage{bm}
\usepackage{url}
\usepackage{natbib}
\usepackage{yfonts}
\usepackage[colorlinks,linkcolor=blue,anchorcolor=red,citecolor=blue,urlcolor=blue]{hyperref}
\usepackage{geometry}
\DeclareMathOperator{\sgn}{sgn}
\DeclareMathOperator{\tr}{tr}
\begin{document}

\title{A Nonlinear Kernel Support Matrix Machine for Matrix Learning
}


\author{Yunfei Ye
}


\institute{Yunfei Ye \at
              Department of Mathematical Sciences, Shanghai Jiao Tong University \\
              800 Dongchuan RD Shanghai, 200240 China \\
              \email{tianshapojun@sjtu.edu.cn}           
}

\date{First Vision: \today} 

\maketitle

\begin{abstract}
In many problems of supervised tensor learning (STL), real world data such as face images or MRI scans are naturally represented as matrices, which are also called as second order tensors. Most existing classifiers based on tensor representation, such as support tensor machine (STM) need to solve iteratively which occupy much time and may suffer from local minima. In this paper, we present a kernel support matrix machine (KSMM) to perform supervised learning when data are represented as matrices. KSMM is a general framework for the construction of matrix-based hyperplane to exploit structural information. We analyze a unifying optimization problem for which we propose an asymptotically convergent algorithm. Theoretical analysis for the generalization bounds is derived based on Rademacher complexity with respect to a probability distribution. We demonstrate the merits of the proposed method by exhaustive experiments on both simulation study and a number of real-word datasets from a variety of application domains.

\end{abstract}

\section{Introduction}
The supervised learning tasks are often encountered in many fields including pattern recognition,
image processing and data mining. Data are represented as feature vectors to handle
such tasks. Among all the algorithms based on the vector framework, Support Vector Machine
(SVM) \citep{vapnik1995} is the most representative one due to numerous theoretical
and computational developments. Later, the support vector method was extended to improve
its performance in many applications. Radial basis function classifiers were introduced in SVM to solve nonlinear separable problems \citep{scholkopf1997comparing}. The use of SVM for density estimation \citep{weston1997density} and ANOVA decomposition \citep{stitson1997support} has also been studied. Least squares SVM (LS-SVM) \citep{suykens1999least} modifies the equality constraints in the optimization problem to solve a set of linear equations instead of quadratic ones. Transductive SVM (TSVM) \citep{joachims1999transductive} tries to minimize misclassification error of a particular test set. $\nu$-SVM \citep{scholkopf2000new} includes a new parameter $\nu$ to effectively control the number of support vectors for both regression and classification. One-Class SVM (OCSVM) \citep{scholkopf2001estimating} aims to identify one available class, while characterizing other classes is either expensive or difficult. Twin SVM (TWSVM) \citep{khemchandani2007twin} is a fast algorithm solving two quadratic programming problems of smaller sizes instead of a large one in classical SVM.

However it is more natural to represent real-world data as matrices or higher-order tensors. Within the last decade, advanced researches have been exploited on retaining the structure of tensor data and extending SVM to tensor patterns. Tao et al. proposed a Supervised Tensor Learning (STL) framework to address the tensor problems \citep{tao2005supervised}. Under this framework, Support Tensor Machine (STM) was studied to separate multilinear hyperplanes by applying alternating projection methods \citep{cai2006learning}. Tao et al. \citep{Tao2007Supervised} extended the classical linear C-SVM \citep{cortes1995support}, $\nu$-SVM and least squares SVM \citep{suykens1999least} to general tensor patterns. One-Class STM (OCSTM) was generalized to obtain most interesting tensor class with maximal margin \citep{chen2016one,erfani2016r1stm}. Joint tensor feature analysis (JTFA) was proposed for tensor feature extraction and recognition by \cite{wong2015joint}. Support Higher-order Tensor Machine (SHTM) \citep{hao2013linear} integrates the merits of linear C-SVM and rank-one decomposition. Kernel methods for tensors were also introduced in nonlinear cases. Factor kernel \citep{signoretto2011kernel} calculates the similarity between tensors using techniques of tensor unfolding and singular value decomposition (SVD). Dual structure-preserving kernels (DuSK) \citep{he2014dusk} is a generalization of SHTM with dual-tensorial mapping functions to detect dependencies of nonlinear tensors. Support matrix machines (SMM) \citep{Luo2015Support} aims to solve a convex optimization problem considering a hinge loss plus a so-called spectral elastic net penalty. These methods essentially take advantage of the low-rank assumption, which can be used for describing the correlation within a tensor.

In this paper we are concerned with classification problems on a set of matrix data. We present a kernel support matrix machine (KSMM) and it is motivated by the use of matrix Hilbert space \citep{2017arXiv170608110Y}. Its cornerstones is the introduction of a matrix as the inner product to compile the complicated relationship among samples. KSMM is a general framework for constructing a matrix-based hyperplane through calculating the weighted average distance between data and multiple hyperplanes. We analyze a unifying optimization problem for which we propose an asymptotically convergent algorithm built on the Sequential Minimal Optimization (SMO) \citep{Platt1999Fast} algorithm.  Generalization bounds of SVM were discussed based on Rademacher complexity with respect to a probability distribution \citep{Shalev2014Understanding}; here we extend their definitions to a more general and flexible framework. The contribution of this paper is listed as follows. One main contribution is to develop a new classifier for matrix learning where the optimization problem is solved directly without adopting the technique of alternating projection method in STL. Important special cases of the framework include classifiers of SVM. Another contribution lies within a matrix-based hyperplane that we propose in the algorithm to separate objects instead of determining multiple hyperplanes as in STM.

The rest of this paper is organized as follows. In Sect.~\ref{lksmm}, we discuss the framework of kernel support matrix machine in linear case. We show its dual problem and present a template algorithm to solve this problem. In Sect.~\ref{nksmm} we extend to the nonlinear task by adopting the methodology of reproducing kernels. Sect.~\ref{gb} deals with the generalization bounds based on Rademacher complexity with respect to a probability distribution. Differences among several classifiers are discussed in Sect.~\ref{discuss}. In Sect.~\ref{experiment} we study our model's performance in both simulation study and benchmark datasets. Finally, concluding remarks are drawn in Sect.~\ref{cr}.

\section{Kernel Support Matrix Machine}

In this section, we put forward the Kernel Support Matrix Machine (KSMM) which makes a closed connection between matrix Hilbert Space \citep{2017arXiv170608110Y} and the supervised tensor learning (STL). We construct a hyperplane in the matrix Hilbert space to separate two communities of examples. Then, the SMO algorithm is introduced to handle with the new optimization problem. Next, we derive the generalization bounds for KSMM based on Rademacher complexity with respect to a probability distribution. Finally, we analyze and compare the differences of KSMM with other state-of-the-art methodologies.

\subsection{Kernel Support Matrix Machine in linear case}\label{lksmm}

We first introduce some basic notations and definitions. In this study, scales are denoted by lowercase letters, e.g., s, vectors by boldface lowercase letters, e.g., \textbf{v}, matrices by boldface capital letters, e.g., \textbf{M} and general sets or spaces by gothic letters, e.g., $\mathcal{S}$.

The Frobenius norm of a matrix $\textbf{X} \in \mathbb{R}^{m\times n}$ is defined by
\begin{equation*}
  \| \textbf{X} \|=\sqrt{\sum_{i_1=1}^{m} \sum_{i_2=1}^{n} x_{i_1 i_2}^2},
\end{equation*}
which is a generalization of the normal $\ell_2$ norm for vectors.

The inner product of two same-sized matrices $\textbf{X},\textbf{Y} \in \mathbb{R}^{m\times n}$ is defined as the sum of products of their entries, i.e.,
\begin{equation*}
  \langle \textbf{X},\textbf{Y} \rangle=\sum_{i_1=1}^{m} \sum_{i_2=1}^{n} x_{i_1 i_2}y_{i_1 i_2}.
\end{equation*}

Inspired by the previous work, we introduce the matrix inner product to the framework of STM in matrix space. The matrix inner product is defined as follows:

\begin{definition}[Matrix Inner Product]\label{ip}
Let $\mathcal{H}=\mathbb{R}^{m \times n}$ be a real linear space, the matrix inner product is a mapping $\langle \cdot, \cdot \rangle_{\mathcal{H}} : \mathcal{H} \times \mathcal{H} \rightarrow \mathbb{R}^{n \times n}$ satisfying the following properties, for all $\textbf{X}, \textbf{X}_1, \textbf{X}_2, \textbf{Y} \in \mathcal{H}$

(1) $\langle \textbf{Y},\textbf{X} \rangle_{\mathcal{H}} = \langle \textbf{X},\textbf{Y} \rangle_{\mathcal{H}} ^\intercal$

(2) $\langle \lambda\textbf{X}_1+\mu\textbf{X}_2,\textbf{Y} \rangle_{\mathcal{H}} = \lambda\langle \textbf{X}_1,\textbf{Y} \rangle_{\mathcal{H}}+\mu\langle \textbf{X}_2,\textbf{Y} \rangle_{\mathcal{H}}$

(3) $\langle \textbf{X},\textbf{X} \rangle_{\mathcal{H}}=\textbf{0}$ if and only if \textbf{X} is a null matrix

(4) $\langle \textbf{X},\textbf{X} \rangle_{\mathcal{H}}$ is positive semidefinite.
\end{definition}



This thus motivates us to reformulate the optimization problem in STM. Considering a set of samples $\{(y_i,\textbf{X}_i)\}_{i=1}^N$ for binary classification problem, where $\textbf{X}_i \in \mathbb{R}^{m \times n}$ are input matrix data and $y_i \in \{-1,+1\}$ are corresponding class labels. We assume that $\{\textbf{X}_i\}_{i=1}^N$ and $\textbf{W} \in \mathbb{R}^{m \times n}$ are in the matrix Hilbert space $\mathcal{H}$, $\textbf{V} \in \mathbb{R}^{n \times n}$ is a symmetric matrix satisfying: 

\begin{equation}\label{ass1}
\langle \langle \textbf{X},\textbf{X} \rangle_{\mathcal{H}}, \frac{\textbf{V}}{\|\textbf{V}\|} \rangle \geq 0 
\end{equation}
for all $\textbf{X} \in \mathcal{H}$. In particular, the problem of KSMM can be described in the following way:
\begin{equation}\label{op}
\begin{split}
  &\min_{\textbf{W},b,\bm{\xi},\textbf{V}} \ \frac{1}{2}\|\textbf{W}\|_{\mathcal{H}}^2(\textbf{V}) + C \sum_{i=1}^N \xi_i \\
  &s.t. \ y_i(\langle\langle \textbf{W},\textbf{X}_i\rangle_{\mathcal{H}},\frac{\textbf{V}}{\|\textbf{V}\|}\rangle+b)\geq 1-\xi_i, \ 1 \leq i \leq N \\
  &\quad \ \ \bm{\xi} \geq 0,
\end{split}
\end{equation}
where the inner product $\langle \cdot,\cdot \rangle_{\mathcal{H}}$ is specified as $\langle \textbf{A},\textbf{B} \rangle_{\mathcal{H}}=\textbf{A}^\intercal \textbf{B}$ for $\textbf{A}, \textbf{B} \in \mathbb{R}^{m \times n}$ and the norm is defined as  $\|\cdot\|_{\mathcal{H}}(\textbf{V})=(\langle \langle \cdot, \cdot \rangle_{\mathcal{H}}, \frac{\textbf{V}}{\|\textbf{V}\|} \rangle)^{1/2}$. $\bm{\xi}=[\xi_1, \cdots, \xi_N]^T$ is the vector of all slack variables of training examples and $C$ is the trade-off between the classification margin and misclassification error. 

\begin{remark}
 The proposed problem (\ref{op}) degenerates into the classical SVM when $n=1$. 
 \end{remark}

\begin{remark}
Two classes of labels are separated by a hyperplane $\langle\langle \textbf{W},\textbf{X}_i\rangle_{\mathcal{H}},\frac{\textbf{V}}{\|\textbf{V}\|}\rangle+b=0$. The expression can be decomposed into two parts: one is controlled by normal matrix \textbf{W} while the other is constrained by weight matrix \textbf{V}. Each entry of the matrix inner product $\langle \textbf{W},\textbf{X}_i\rangle_{\mathcal{H}}$ measures a relative ``distance'' from \textbf{X} to a certain hyperplane. To make explicit those values underlying their own behavior, we introduce a weight matrix \textbf{V} which determines the relative importance of each hyperplane on the average.
\end{remark}

Once the model has been solved, the class label of a testing example X can be predicted as follow:
\begin{equation}
y(\textbf{X})=\sgn(\langle\langle \textbf{W},\textbf{X}\rangle_{\mathcal{H}},\frac{\textbf{V}}{\|\textbf{V}\|}\rangle+b).
\end{equation}

The Lagrangian function of the optimization problem (\ref{op}) is
\begin{equation}\label{Lp}
L(\textbf{W},b,\bm{\xi},\textbf{V})=\frac{1}{2}\|\textbf{W}\|_{\mathcal{H}}^2(\textbf{V}) + C \sum_{i=1}^N \xi_i-\sum_{i=1}^N \alpha_{i}(y_i(\langle\langle \textbf{W},\textbf{X}_i\rangle_{\mathcal{H}},\frac{\textbf{V}}{\|\textbf{V}\|}\rangle+b)-1+\xi_i)
-\sum_{i=1}^N \beta_{i}\xi_{i}.
\end{equation}

Let the partial derivatives of $L(\textbf{W},b,\bm{\xi},\textbf{V})$ with respect to
\textbf{W}, b, $\bm{\xi}$ and \textbf{V} be zeros respectively, we have
\begin{equation}\label{Lc}
\begin{split}
& \textbf{W}= \sum_{i=1}^N \alpha_i y_i \textbf{X}_i. \\
& \sum_{i=1}^N \alpha_i y_i=0. \\
& \alpha_i+\beta_i=C, \ i=1,\cdots,N. \\
& \textbf{V}= a\sum_{i,j=1}^N \alpha_i \alpha_j y_i y_j \langle\textbf{X}_i, \textbf{X}_j\rangle_{\mathcal{H}},
\end{split}
\end{equation}
where $a$ is a positive real number. Substituting (\ref{Lc}) into (\ref{Lp}) yields the dual of the optimization problem (\ref{op}) as follows:
\begin{equation}\label{dp}
\begin{split}
  &\max_{\alpha} \ \sum_{i=1}^N \alpha_i-\frac{1}{2}(\langle \sum_{i,j=1}^N \alpha_i \alpha_j y_i y_j \langle\textbf{X}_i, \textbf{X}_j\rangle_{\mathcal{H}},\sum_{i,j=1}^N \alpha_i \alpha_j y_i y_j \langle\textbf{X}_i, \textbf{X}_j\rangle_{\mathcal{H}}\rangle)^{1/2} \\
  &s.t. \ \sum_{i=1}^N \alpha_i y_i=0, \\
  &\quad \ \ 0\leq \alpha_i \leq C, \ 1 \leq i \leq N
\end{split}
\end{equation}
where $\alpha_i$ are the Lagrange multipliers.

Notice that for all $\textbf{X} \in \mathcal{H}$, we have
\begin{equation}
\begin{split}
\langle \langle \textbf{X},\textbf{X} \rangle_{\mathcal{H}}, \frac{\textbf{V}}{\|\textbf{V}\|} \rangle &= \langle \textbf{X}^{\intercal} \textbf{X},\frac{\textbf{W}^{\intercal} \textbf{W}}{\|\textbf{W}^{\intercal} \textbf{W}\|}\rangle =\frac{1}{\|\textbf{W}^{\intercal} \textbf{W}\|}\sum_{i,j=1}^m (\sum_{p=1}^m x_{pi} x_{pj}) (\sum_{q=1}^m w_{qi} w_{qj}) \\
&=\frac{1}{\|\textbf{W}^{\intercal} \textbf{W}\|} \sum_{p,q=1}^m (\sum_{i=1}^n x_{pi} w_{qi}) (\sum_{j=1}^n x_{pj} w_{qj}) \\
&=\frac{1}{\|\textbf{W}^{\intercal} \textbf{W}\|} \sum_{p,q=1}^m (\sum_{i=1}^n x_{pi} w_{qi})^2 \geq 0,
\end{split}
\end{equation}
which indicates that the matrix \textbf{V} we derives from Lagrange multiplier method satisfies condition (\ref{ass1}).

Furthermore, the Karush-Kuhn-Tucker (KKT) conditions are fulfilled when the optimization problem (\ref{Lp}) is solved, that is for all $i$:
\begin{equation}\label{KTT}
\begin{split}
\alpha_i=0 &\Rightarrow y_i f(\textbf{X}_i) \geq 1, \\
0 < \alpha_i < C &\Rightarrow y_i f(\textbf{X}_i) = 1, \\
\alpha_i=C &\Rightarrow y_i f(\textbf{X}_i) \leq 1, \\
\end{split}
\end{equation}
where $f(\textbf{X}_i)=\langle\langle \textbf{W},\textbf{X}_i\rangle_{\mathcal{H}},\frac{\textbf{V}}{\|\textbf{V}\|}\rangle+b$. Next, we summarize and improve the SMO algorithm to solve the optimization problem (\ref{dp}). At each step, SMO chooses two Lagrange multipliers to jointly optimize the objective function $J(\bm{\alpha})$ while other multipliers are fixed, which can be computed as follow:
\begin{equation}\label{of}
J(\bm{\alpha})=\sum_{i=1}^N \alpha_i-\frac{1}{2}(\langle \sum_{i,j=1}^N \alpha_i \alpha_j y_i y_j \langle\textbf{X}_i, \textbf{X}_j\rangle_{\mathcal{H}},\sum_{i,j=1}^N \alpha_i \alpha_j y_i y_j \langle\textbf{X}_i, \textbf{X}_j\rangle_{\mathcal{H}}\rangle)^{1/2}.
\end{equation}

For convenience, all quantities that refer to the first multiplier will have a subscript 1, while all quantities that refer to the second multiplier will have a subscript 2. Without lose of generality, the algorithm computes the second Lagrange multiplier $\alpha_2$ and then updates the first Lagrange multiplier $\alpha_1$ at each step. Notice that $\alpha_1 y_1+\alpha_2 y_2=constant \Leftrightarrow \alpha_1=constant-y_1 y_2 \alpha_2$, we can rewritten (\ref{of}) in terms of $\alpha_2$ as:
\begin{equation*}
J(\alpha_2)= (1-y_1 y_2)\alpha_2-\frac{1}{2}(\langle \textbf{W}^\intercal\textbf{W},\textbf{W}^\intercal\textbf{W} \rangle)^{1/2}+constant,
\end{equation*}
where $\textbf{W}= \sum_{i=1}^N \alpha_i y_i \textbf{X}_i$ and $\frac{\partial \textbf{W}}{\partial \alpha_2}=y_2(\textbf{X}_2-\textbf{X}_1)=\textbf{A}$.

Compute the partial derivative and second partial derivative of the object function, we can obtain that
\begin{equation}\label{pd}
\begin{split}
& \frac{\partial J}{\partial \alpha_2}=1-y_1 y_2-\frac{\langle \textbf{A}^\intercal\textbf{W},\textbf{W}^\intercal\textbf{W}\rangle}{(\langle \textbf{W}^\intercal\textbf{W},\textbf{W}^\intercal\textbf{W} \rangle)^{1/2}}, \\
& \frac{\partial^2 J}{\partial \alpha_2^2}=-\frac{(\langle \textbf{A}^\intercal\textbf{A},\textbf{W}^\intercal\textbf{W} \rangle+\langle \textbf{A}^\intercal\textbf{W},\textbf{A}^\intercal\textbf{W}+\textbf{W}^\intercal\textbf{A} \rangle)\langle \textbf{W}^\intercal\textbf{W},\textbf{W}^\intercal\textbf{W} \rangle-2\langle \textbf{A}^\intercal\textbf{W},\textbf{W}^\intercal\textbf{W}\rangle^2}{(\langle \textbf{W}^\intercal\textbf{W},\textbf{W}^\intercal\textbf{W} \rangle)^{3/2}}.
\end{split}
\end{equation}

We can easily derive that
\begin{equation*}
\begin{split}
& \langle \textbf{A}^\intercal\textbf{A},\textbf{W}^\intercal\textbf{W} \rangle=\langle \langle \textbf{A},\textbf{A}\rangle_{\mathcal{H}},a\textbf{V}\rangle=a\|\textbf{A}\|_{\mathcal{H}}^2 \|\textbf{V}\| \geq 0, \\
&\langle \textbf{A}^\intercal\textbf{W},\textbf{A}^\intercal\textbf{W}+\textbf{W}^\intercal\textbf{A} \rangle \langle \textbf{W}^\intercal\textbf{W},\textbf{W}^\intercal\textbf{W} \rangle-2\langle \textbf{A}^\intercal\textbf{W},\textbf{W}^\intercal\textbf{W}\rangle^2 \\
& = \frac{1}{2} \langle\textbf{A}^\intercal\textbf{W}+\textbf{W}^\intercal\textbf{A},\textbf{A}^\intercal\textbf{W}+\textbf{W}^\intercal\textbf{A} \rangle \langle \textbf{W}^\intercal\textbf{W},\textbf{W}^\intercal\textbf{W} \rangle-\frac{1}{2}\langle \textbf{A}^\intercal\textbf{W}+\textbf{W}^\intercal\textbf{A},\textbf{W}^\intercal\textbf{W}\rangle^2 \geq 0.
\end{split}
\end{equation*}
The second inequality holds according to the Cauchy-Schwarz inequality. The second partial derivative of the objective function is no more than zero. Therefore, the location of the constrained maximum of the objective function is either at the bounds or at the extreme point.

On the other hand, let $\frac{\partial J}{\partial \alpha_2}$ be zero we obtain a function of the sixth degree which does not have a closed-form. Therefore, the Newton's method is applied to iteratively find the optimal value of $\alpha_2$. At each step, we update the $\alpha_2^{n+1}$ as:
\begin{equation}
\alpha_2^{n+1}=\alpha_2^{n}-\frac{J''(\alpha_2^{n})}{J'(\alpha_2^{n})}, \ \alpha_2^{0}=\alpha_2^{old}
\end{equation}
until it converges to $\alpha_2^{new}$.

Remember that the two Lagrange multipliers must fulfill all of the constraints of problem (\ref{Lp}) that the lower bound $L$ and the upper bound $H$ of $\alpha_2$ can be concluded as for labels $y_1 \neq y_2$:
\begin{equation}
L=\max(0,\alpha_2^{old}-\alpha_1^{old}), \ H=\min(C,C+\alpha_2^{old}-\alpha_1^{old}).
\end{equation}
If labels $y_1 = y_2$, then the following bounds apply to $\alpha_2$:
\begin{equation}
L=\max(0,\alpha_1^{old}+\alpha_2^{old}-C), \ H=\min(C,\alpha_1^{old}+\alpha_2^{old}).
\end{equation}

Next, the constrained maximum is found by clipping the unconstrained maximum to the bounds of the domain:
\begin{equation}\label{cc}
\alpha_2^{new,clipped}=\left\{
\begin{aligned}
&H, &if \ \alpha_2^{new} \geq H\\
&\alpha_2^{new}, &if \ L \leq \alpha_2^{new} \leq H\\
&L, &if \ \alpha_2^{new} \leq L.
\end{aligned}
\right.
\end{equation}

Then the value of $\alpha_1$ is calculated from the new, clipped $\alpha_2$:
\begin{equation}
\alpha_1^{new}=\alpha_1^{old}+y_1 y_2 (\alpha_2^{old}-\alpha_2^{new,clipped}).
\end{equation}

This process is repeated iteratively until the maximum number of outer loops M is reached or all of the Lagrange multipliers hold the KTT conditions. Typically, we terminate the inner loop of Newton's method if $\|\frac{J''(\alpha_2^{old})}{J'(\alpha_2^{old})}\|<\varepsilon$, where $\varepsilon$ is a threshold parameter. 

Then we present the strategy on the choices of two Lagrange multipliers. When iterates over the entire training set, the first one which violates the KTT condition (\ref{KTT}) is determined as the first multiplier. Once a violated example is found, the second multiplier is chosen randomly unlike that of the classical SMO for the closed-form of the extreme point can not be derived directly.

\begin{table}
\begin{tabular*}{1\textwidth}{@{\extracolsep{\fill}}  l  }
\hline\noalign{\smallskip}
\textbf{Algorithm:} Kernel Support Matrix Machine \\
\hline
\noalign{\smallskip}

\noindent \textbf{Input:} The set of training data $\{\textbf{X}_i \in \mathbb{R}^{m \times n},y_i\}_{i=1}^N$, test data set $\{\textbf{Z}_i\}$, cost C,  maximum
number of \\ outer loops M and threshold parameter $\varepsilon$ \\

\noindent \textbf{Output:} The estimated label $y(\textbf{Z}_i)$ \\

\noindent Initialization. Take $t=0$, $\alpha_1^0=\cdots=\alpha_N^0 \in [0,C]$ \\

\noindent \textbf{while} Stopping criterion is not satisfied \textbf{do}\\

\noindent \qquad Get $\alpha_1^t$ which validates condition (\ref{KTT})\\

\noindent \qquad  Randomly pick up $\alpha_2^t$ \\

\noindent \qquad  $\alpha_2^{old}=\alpha_2^{new}=\alpha_2^t$ \\

\noindent \qquad \textbf{while} Stopping criterion is not satisfied \textbf{do}\\

\noindent \qquad \qquad $\alpha_2^{old}=\alpha_2^{new}$ \\

\noindent \qquad \qquad $\alpha_2^{new}=\alpha_2^{old}-\frac{J''(\alpha_2^{old})}{J'(\alpha_2^{old})}$ \\

\noindent \qquad \textbf{end while} \\

\noindent \qquad $\alpha_2^{t+1}=\alpha_2^{new,clipped}$ using (\ref{cc})\\

\noindent \qquad $\alpha_1^{t+1}=\alpha_1^{t}+y_1 y_2 (\alpha_2^{t}-\alpha_2^{t+1})$ \\

\noindent \qquad $t \leftarrow t+1$ \\

\noindent \textbf{end while} \\

\noindent $\widehat{\bm{\alpha}} \leftarrow \bm{\alpha}^{t}$ \\

\noindent Calculate $\widehat{b}$ in problem (\ref{op}) \\

\noindent $y(\textbf{Z}_j)=\sgn(\langle \sum_{i=1}^N \widehat{\alpha}_i y_i \langle \textbf{X}_i,\textbf{Z}_j\rangle_{\mathcal{H}}, \frac{\sum_{i,j=1}^N \widehat{\alpha}_i \widehat{\alpha}_j y_i y_j \langle\textbf{X}_i, \textbf{X}_j\rangle_{\mathcal{H}}}{\|\sum_{i,j=1}^N \widehat{\alpha}_i \widehat{\alpha}_j y_i y_j \langle\textbf{X}_i, \textbf{X}_j\rangle_{\mathcal{H}}\|} \rangle+\widehat{b})$ \\
\hline
\end{tabular*}
\end{table}

Like the SMO algorithm, we update the parameter $b$ using following strategy: the parameter $b_2$ updates when the new $\alpha_2$ is not at the bounds which forces the output $y(\textbf{X}_2)$ to be $y_2$.
\begin{equation*}
b_2=y_2-\langle \sum_{i=1}^N \widehat{\alpha}_i y_i \langle \textbf{X}_i,\textbf{X}_2\rangle_{\mathcal{H}}, \frac{\sum_{i,j=1}^N \widehat{\alpha}_i \widehat{\alpha}_j y_i y_j \langle\textbf{X}_i, \textbf{X}_j\rangle_{\mathcal{H}}}{\|\sum_{i,j=1}^N \widehat{\alpha}_i \widehat{\alpha}_j y_i y_j \langle\textbf{X}_i, \textbf{X}_j\rangle_{\mathcal{H}}\|} \rangle.
\end{equation*}
The parameter $b_1$ updates when the new $\alpha_1$ is not at the bounds which forces the output $y(\textbf{X}_1)$ to be $y_1$.
\begin{equation*}
b_1=y_1-\langle \sum_{i=1}^N \widehat{\alpha}_i y_i \langle \textbf{X}_i,\textbf{X}_1\rangle_{\mathcal{H}}, \frac{\sum_{i,j=1}^N \widehat{\alpha}_i \widehat{\alpha}_j y_i y_j \langle\textbf{X}_i, \textbf{X}_j\rangle_{\mathcal{H}}}{\|\sum_{i,j=1}^N \widehat{\alpha}_i \widehat{\alpha}_j y_i y_j \langle\textbf{X}_i, \textbf{X}_j\rangle_{\mathcal{H}}\|} \rangle.
\end{equation*}

When both $b_1$ and $b_2$ are updated, they are equal. When both Lagrange multipliers are at the bounds, any number in the interval between $b_1$ and $b_2$ is consistent with the KKT conditions. We choose the threshold to be the average of $b_1$ and $b_2$. The Pseudo-code of the overall algorithm is listed above.

The objective function increases at every step and the algorithm will converge asymptotically. Even though the extra Newton's method is applied in each iteration, the overall algorithm does work efficiently.

\subsection{Kernel Support Matrix Machine in nonlinear case}\label{nksmm}
Kernel methods, which refer to as ``kernel trick'' were brought to the field of machine learning in the 20th century to overcome the difficulty in detecting certain dependencies of nonlinear problems. A Reproducing Kernel Matrix Hilbert Space (RKMHS) \citep{2017arXiv170608110Y} was introduced to develop kernel theories in the matrix Hilbert space. In this section, we define a nonlinear mapping and apply these algorithms in our KSMM.

We start by defining the following mapping on a matrix $\textbf{X} \in \mathbb{R}^{m \times n}$.
\begin{equation}
\Phi: \textbf{X} \mapsto \Phi(\textbf{X}) \in \mathbb{R}^{m \times n},
\end{equation}
where $\Phi(\textbf{X})$ is in a matrix Hilbert space $\mathcal{H}'$. Naturally, the kernel function is defined as inner products of elements in the feature space:
\begin{equation}\label{kf}
K(\textbf{X}_i,\textbf{X}_j)=\langle \Phi(\textbf{X}_i),\Phi(\textbf{X}_j) \rangle_{\mathcal{H}'} \in \mathbb{R}^{n \times n}.
\end{equation}
Further details of the structure of a RKMHS can be found in \cite{2017arXiv170608110Y}.

Substituting (\ref{kf}) into (\ref{dp}) with mapping $\Phi$ yieds the nonlinear problem as follows:
\begin{equation}
\begin{split}
  &\max_{\alpha} \ \sum_{i=1}^N \alpha_i-\frac{1}{2}(\langle \sum_{i,j=1}^N \alpha_i \alpha_j y_i y_j K(\textbf{X}_i, \textbf{X}_j),\sum_{i,j=1}^N \alpha_i \alpha_j y_i y_j K(\textbf{X}_i, \textbf{X}_j)\rangle)^{1/2} \\
  &s.t. \ \sum_{i=1}^N \alpha_i y_i=0, \\
  &\quad \ \ 0\leq \alpha_i \leq C, \ 1 \leq i \leq N
\end{split}
\end{equation}

The revised SMO algorithm is still applied under such circumstance. We emphasize that $\textbf{W}= \sum_{i=1}^N \alpha_i y_i \Phi(\textbf{X}_i)$ and $\frac{\partial \textbf{W}}{\partial \alpha_2}=y_2(\Phi(\textbf{X}_2)-\Phi(\textbf{X}_1))=\textbf{A}$. The following abbreviations are derived to compute the partial derivative and second partial derivative of the object function $J(\bm{\alpha})$ in (\ref{pd}).
\begin{equation}
\begin{split}
& \textbf{A}^\intercal\textbf{A}= K(\textbf{X}_2,\textbf{X}_2)+K(\textbf{X}_1,\textbf{X}_1)-K(\textbf{X}_1,\textbf{X}_2)-K(\textbf{X}_2,\textbf{X}_1), \\
& \textbf{A}^\intercal\textbf{W}=\sum_{i=1}^N \alpha_i y_i y_2 (K(\textbf{X}_2,\textbf{X}_i)-K(\textbf{X}_1,\textbf{X}_i)), \\
& \textbf{W}^\intercal\textbf{W}=\sum_{i,j=1}^N \alpha_i \alpha_j y_i y_j K(\textbf{X}_i,\textbf{X}_j).
\end{split}
\end{equation}

Some possible choices of $K$ include
\begin{equation*}
\begin{split}
&\rm{Linear \ kernel:} \qquad K(\textbf{X},\textbf{Y})=\textbf{X}^\intercal \textbf{Y}+\alpha \textbf{I}_{n \times n}, \\
&\rm{Polynomial \ kernel:} \quad K(\textbf{X},\textbf{Y})= (\textbf{X}^\intercal \textbf{Y}+\alpha \textbf{I}_{n \times n})^{\circ \beta}, \\
&\rm{Gaussian \ kernel:} \quad  K(\textbf{X},\textbf{Y})=[\exp(-\gamma \|\textbf{X}(:,i)-\textbf{Y}(:,j)\|^2)]_{n \times n},
\end{split}
\end{equation*}
where $\alpha\geq 0, \beta \in \mathbb{N}, \gamma >0, \textbf{X}, \textbf{Y} \in \mathcal{H}=\mathbb{C}^{m \times n}$. $\textbf{X}(:,i)$ is the $i$-th column of \textbf{X} and $\circ$ is the Hadamard product \citep{horn1990hadamard}.

Additionally, if $\Phi$ is an identical mapping with $K(\textbf{X}_i,\textbf{X}_j)= \textbf{X}_i^\intercal\textbf{X}_j$, the optimization problem will degenerate into a linear one. 

\subsection{Generalization Bounds for KSMM}\label{gb}
In this section, we use Rademacher complexity to obtain generalization bounds for soft-SVM and STM with Frobenius norm constraint. We will show how this leads to generalization bounds for KSMM.

To simplify the notation, we denote
\begin{equation*}
\mathcal{F}= \ell \circ \mathcal{H}_p=\{ z \mapsto \ell(h,z): z \in \mathcal{Z}, h \in \mathcal{H}_p\},
\end{equation*}
where $\mathcal{Z}$ is a domain,  $\mathcal{H}_p$ is a hypothesis class and $\ell$ is a loss function. Given $f \in \mathcal{F}$, we define
\begin{equation*}
L_{\mathcal{D}}(f)= \mathbb{E}_{z \thicksim \mathcal{D}} [f(z)], \quad L_{\mathcal{S}}(f)=\frac{1}{N} \sum_{i=1}^N f(z_i),
\end{equation*}
where $\mathcal{D}$ is the distribution of elements in $\mathcal{Z}$, $\mathcal{S}$ is the training set and $N$ is the number of examples in $\mathcal{S}$.

We repeat the symbols and assumptions in Sect \ref{lksmm} for further study. A STM problem in matrix space can be reformulated as:
\begin{equation}
\begin{split}
  &\min_{\textbf{W},b,\bm{\xi},} \ \frac{1}{2}\|\textbf{W}\|^2 + C \sum_{i=1}^N \xi_i \\
  &s.t. \ y_i(\langle \textbf{W},\textbf{X}_i\rangle+b)\geq 1-\xi_i, \ 1 \leq i \leq N \\
  &\quad \ \ \bm{\xi} \geq 0,
\end{split}
\end{equation}
where $\textbf{W}, \{\textbf{X}_i\}_{i=1}^N \in \mathbb{R}^{m \times n}$ and $\|\cdot\|$ is the Frobenius norm.

Consider the vector as a specialization of matrix that the number of its row or column is equal to one, we rewrite the theorem from \cite{Shalev2014Understanding} in the following way. It bounds the generalization error for SVM and STM(for matrix data) of all predictors in $\mathcal{H}_p$ using their empirical error.

\begin{theorem}\label{stmb}
Suppose that $\mathcal{D}$ is a distribution over $\mathcal{X} \times \mathcal{Y}$ such that with probability 1 we have that $\|\emph{\textbf{X}}\| \leq R$. Let $\mathcal{H}_p=\{\emph{\textbf{W}}: \|\emph{\textbf{W}}\| \leq B\}$ and let $\ell: \mathcal{H}_p \times \mathcal{Z} \rightarrow \mathbb{R}$ be a loss function of the form
\begin{equation*}
\ell (\emph{\textbf{W}},(\emph{\textbf{X}},y))=\Phi(\langle \emph{\textbf{W}},\emph{\textbf{X}}\rangle,y),
\end{equation*}
such that for all $y \in \mathcal{Y}$, $a \mapsto \Phi(a,y)$ is a $\rho$-Lipschitz function and $\max_{a \in [-BR,BR]} |\Phi(a,y)| \leq c$. Then, for any $\delta \in (0,1)$, with probability of at least $1-\delta$ over the choice of an i.i.d. sample of size N,
\begin{equation}
\forall \ \emph{\textbf{W}} \in \mathcal{H}_p, L_{\mathcal{D}}(\emph{\textbf{W}}) \leq L_{\mathcal{S}}(\emph{\textbf{W}})+\frac{2 \rho B R}{\sqrt{N}}+c\sqrt{\frac{2 \ln(2/\delta)}{N}}.
\end{equation}
\end{theorem}

\begin{remark}
When $m=1$ or $n=1$, the matrix $\textbf{X} \in \mathbb{R}^{m \times n}$ transforms into a vector and its Frobenius norm is consistent with the corresponding Euclidean norm. The optimization problems in both classifiers are identical.
\end{remark}

In the case of KSMM, we have the following result where we denote by $\|\cdot\|_{\mathcal{H}}(\textbf{V})=(\langle \langle \cdot, \cdot \rangle_{\mathcal{H}}, \frac{\textbf{V}}{\|\textbf{V}\|} \rangle)^{1/2}$.

\begin{theorem}\label{ksmmb}
Suppose that $\mathcal{D}$ is a distribution over $\mathcal{H} \times \mathcal{Y}$ where $\mathcal{H}$ is a matrix Hilbert space such that with probability 1 we have that $\|\emph{\textbf{X}}\|_{\mathcal{H}}(\emph{\textbf{V}}) \leq R'$. Let $\mathcal{H}'_p=\{\emph{\textbf{W}}': \|\emph{\textbf{W}}'\|_{\mathcal{H}}(\emph{\textbf{V}}) \leq B'\}$ and let $\ell: \mathcal{H}'_p \times \mathcal{Z} \rightarrow \mathbb{R}$ be a loss function of the form
\begin{equation*}
\ell (\emph{\textbf{W}}',(\emph{\textbf{X}},y))=\Phi(\langle \langle \emph{\textbf{W}}',\emph{\textbf{X}}\rangle_{\mathcal{H}},\frac{\emph{\textbf{V}}}{\|\emph{\textbf{V}}\|}\rangle,y),
\end{equation*}
such that for all $y \in \mathcal{Y}$, $a \mapsto \Phi(a,y)$ is a $\rho$-Lipschitz function and $\max_{a \in [-B'R',B'R']} |\Phi(a,y)| \leq c'$. Then, for any $\delta \in (0,1)$, with probability of at least $1-\delta$ over the choice of an i.i.d. sample of size N,
\begin{equation}
\forall \ \emph{\textbf{W}}' \in \mathcal{H}'_p, L_{\mathcal{D}}(\emph{\textbf{W}}') \leq L_{\mathcal{S}}(\emph{\textbf{W}}')+\frac{2 \rho B' R'}{\sqrt{N}}+c'\sqrt{\frac{2 \ln(2/\delta)}{N}}.
\end{equation}
\end{theorem}

\begin{proof}
See Appendix \ref{Aa}.
\end{proof}
\qed

The following theorem compare the generalization bounds with the same hinge-loss function $\Phi(a,y)=\max\{0,1-ay\}$.

\begin{theorem}\label{bd}
In the same domain of $\emph{\textbf{X}} \in \mathcal{X}$ and $\emph{\textbf{W}} \in \mathcal{H}_p$, we have $R' \leq R$, $B' \leq B$ and $c' \leq c$.
\end{theorem}
\begin{proof}
For any $\textbf{X} \in \mathbb{R}^{m \times n}$,
\begin{equation}\label{HleF}
\begin{split}
\|\textbf{X}\|_{\mathcal{H}}^2(\textbf{V}) &=\langle\langle \textbf{X},\textbf{X} \rangle_{\mathcal{H}},\frac{\textbf{V}}{\|\textbf{V}\|}\rangle=\langle\textbf{X}^{\intercal}\textbf{X},\frac{\textbf{V}}{\|\textbf{V}\|}\rangle=\frac{1}{\|\textbf{V}\|}\sum_{p,q=1}^n \sum_{i=1}^m x_{i p} x_{i q} v_{p q} \\
& \leq \frac{1}{\|\textbf{V}\|} \sqrt{(\sum_{p,q=1}^n (\sum_{i=1}^m x_{i p} x_{i q})^2)(\sum_{p,q=1}^n v_{p q}^2)}= \sqrt{(\sum_{p,q=1}^n (\sum_{i=1}^m x_{i p} x_{i q})^2)} \\
& \leq  \sqrt{(\sum_{p,q=1}^n (\sum_{i=1}^m x_{i p}^2)  (\sum_{i=1}^m x_{i q}^2))}= \sqrt{(\sum_{p=1}^n \sum_{i=1}^m x_{i p}^2)  (\sum_{q=1}^n \sum_{i=1}^m x_{i q}^2)} \\
& =\|\textbf{X}\|^2.
\end{split}
\end{equation}
Thus, $R' \leq R$ and so does $B' \leq B$. With $R' \leq R, B' \leq B$, we have $[-B'R',B'R'] \subseteq [-BR,BR]$ and $c' \leq c$.
\end{proof}
\qed

Theorem \ref{bd} suggests that under the same probability, the difference between the true error and empirical error of KSMM is smaller than that of STM. In other words, if we pick up a moderate kernel with better performance on training set within our method, it is more likely to predict a better result on the test step.

On the other hand, normally we do not obtain prior knowledge of the space $\mathcal{H}$, especially for the choice of matrix \textbf{V}. We consider the following general 1-norm  and max norm constraint formulation for matrices where $\|\textbf{X}\|_1=\max\limits_{1 \leq j \leq n} \sum\limits_{i=1}^m |x_{ij}|$ and $\|\textbf{X}\|_{\max}=\max\limits_{1 \leq i,j \leq n} |x_{ij}|$ for $\textbf{X} \in \mathbb{R}^{m \times n}$. The following theorem bounds the generalization error of all predictors in $\mathcal{H}_p$ using their empirical error.

\begin{theorem}\label{ksmmb2}
Suppose that $\mathcal{D}$ is a distribution over $\mathcal{H} \times \mathcal{Y}$ where $\mathcal{H}$ is a matrix Hilbert space such that with probability 1 we have that $\|\emph{\textbf{X}}\|_1 \leq R$. Let $\mathcal{H}_p=\{\emph{\textbf{W}}: \|\emph{\textbf{W}}\|_{\max} \leq B\}$ and let $\ell: \mathcal{H}_p \times \mathcal{Z} \rightarrow \mathbb{R}$ be a loss function of the form
\begin{equation*}
\ell (\emph{\textbf{W}},(\emph{\textbf{X}},y))=\Phi(\langle \langle \emph{\textbf{W}},\emph{\textbf{X}}\rangle_{\mathcal{H}},\frac{\emph{\textbf{V}}}{\|\emph{\textbf{V}}\|}\rangle,y),
\end{equation*}
such that for all $y \in \mathcal{Y}$, $a \mapsto \Phi(a,y)$ is a $\rho$-Lipschitz function and $\max_{a \in [-BR,BR]} |\Phi(a,y)| \leq c$. Then, for any $\delta \in (0,1)$, with probability of at least $1-\delta$ over the choice of an i.i.d. sample of size N,
\begin{equation}
\forall \ \emph{\textbf{W}} \in \mathcal{H}_p, L_{\mathcal{D}}(\emph{\textbf{W}}) \leq L_{\mathcal{S}}(\emph{\textbf{W}})+2 \rho B R n \sqrt{\frac{2(m\ln2+\ln n)}{N}}+c\sqrt{\frac{2 \ln(2/\delta)}{N}}.
\end{equation}
\end{theorem}

\begin{proof}
See Appendix \ref{Ab}.
\end{proof}
\qed

Therefore, we have two bounds given in Theorem \ref{ksmmb} and Theorem \ref{ksmmb2} of KSMM. Apart from the extra $n\ln(n)$ factor, they look in a similar way. These two theorems are constrained to different prior knowledge, one captures low $\mathcal{H}$-norm assumption while the latter is limited to low max norm on \textbf{W} and low 1-norm on \textbf{X}. Note that there is no limitation on the dimension of \textbf{W} to derive the bounds in which kernel methods can be naturally applied.

\subsection{Analysis of KSMM versus other methods}\label{discuss}

We discuss the differences of MRMLKSVM \citep{gao2015multiple}, SVM, STM, SHTM \citep{hao2013linear}, DuSK \citep{he2014dusk} and our new method as follows:

DuSK, which uses CP decomposition and a dual-tensorial mapping to derive a tensor kernel, is a generalization of SHTM. KSMM constructs a matrix-based hyperplane with Newton's method applied in the process of seeking appropriate parameters. All the optimization problems mentioned above only need to be solved once. Based on the alternating projection method, STM, MRMLKSVM need to be solved iteratively, which consume much more time. For a set of matrix samples $\{\textbf{X}_i \in \mathbb{R}^{m \times n},y_i\}_{i=1}^N$, the memory space occupied by SVM is $O((N+1)mn+1)$, STM requires $O(Nmn+m+n+1)$, DuSK requires $O((N+1) r (m+n)+1)$, MRMLKSVM requires $O(Nmn+r(m+n)+1)$ and KSMM requires $O((N+2)mn+1)$, where $r$ is the rank of matrix. KSMM calculates weight matrix \textbf{V} to determine the relative importance of each hyperplane on the average.

Naturally STM is a multilinear support vector machine using different hyperplanes to separate the projections of data points. KSMM is a nonlinear supervised tensor learning and construct a single hyperplane in the matrix Hilbert space.

From the previous work \citep{chu2007map}, we know that the computational complexity of SVM is $O(N^2 mn)$, thus STM is $O(2N^2 T_1 mn)$, DuSK is $O(N^2 r^2(m+n))$, MRMLKSVM is $O(2N^2 T_2 n^2)$, while the complexity of KSMM is $O(N^2 P mn^2)$, where $\{T_i\}_{i=1,2}$ is the corresponding number of iterations and $P$ is the average number of iterations of Newton's method, which is usually small in practice. Moreover, its complexity can be narrowed for the optimal time complexity of multiplication of square matrices has been $O(n^{2.3728639})$ up to now \citep{le2014powers}.

\section{Experiments}\label{experiment}

In this section, we conduct one simulation study on synthetic data and four experiments on benchmark datasets. We validate the effectiveness of KSMM with other methodologies (DuSK \citep{he2014dusk}, Gaussian-RBF, matrix kernel \citep{Gao2014NLS} on SVM or STM classifier, SMM \citep{Luo2015Support}), since they have been proven successful in various applications.

We introduce two comparison of methods to verify our claims about the improvement of the proposed approach. We report the accuracy which counts on the proportion of correct predictions, $F_1=2 \cdot \frac{Pre \times Rec}{Pre+Rec}$ as the harmonic mean of precision and recall. Precision is the fraction of retrieved instances that are relevant, while recall is the fraction of relevant instances that are retrieved. In multiple classification problems, macro-averaged F-measure \citep{yang1999re} is adopted as the average of $F_1$ score for each category.

All experiments were conducted on a computer with Intel(R) Core(TM) i7 (2.50 GHZ) processor with 8.0 GB RAM memory. The algorithms were implemented in Matlab.

\subsection{Simulation study}
In order to get better insight of the proposed approach, we focus on the behavior of proposed methods for different attributes and given examples in binary classification problems. Datasets are subject to the Wishart distribution defined over symmetric, positive-definite matrix-valued random variables,  which is a generalization to multiple dimensions of the chi-squared distribution. Its probability density function is given by
\begin{equation*}
f(\textbf{X})=\frac{1}{2^{np/2} |\textbf{A}|^{n/2} \Gamma_p|\frac{n}{2}|} |\textbf{X}|^{(n-p-1)/2} e^{-\tr(\textbf{A}^{-1}\textbf{X})/2},
\end{equation*}
where \textbf{X} and $\textbf{A}$ are $p \times p$ symmetric, positive-definite matrices, $n$ is the number of degrees of freedom greater than $p-1$ and $\Gamma_p$ is the multivariate gamma function. The problem is verified with the following set-ups:

It is assumed that the considered objects are described by $10 \times 10, 20 \times 20, 30 \times 30, 40 \times 40$ and $50 \times 50$ matrices respectively. The attributes are generated independently with the Wishart distribution with $\textbf{A}=\textbf{u} \textbf{u}^{\intercal}, \textbf{u} \sim \mathcal{N}(\textbf{0},\textbf{I}_p)$, $n=p$ for the first class and $n=2p$ for the second class, for $p=10,\cdots,50$. Additional Gaussian white noise is considered while evaluation is performed with $N=100$ and 200 examples, half of which are selected as a training set while other examples are organized as a test set. For each setting we average results over 10 trials each of which is obtained from the proposed distribution. The input matrices are converted into vectors when it comes to SVM problems. All the kernels select the optimal trade-off parameter from $C \in \{10^{-2},10^{-1},\cdots,10^2\}$, kernel width parameter from $\sigma \in \{10^{-4},10^{-3},\cdots,10^4\}$ and rank from $r \in \{1,2,\cdots,10\}$. All the learning machines use the same training and test set. We first randomly sample $25\%$ of whole data from each dataset for the purpose of parameter selection. Gaussian RBF kernels are used on all MRMLKSVM, DuSK and SVM which denoted as $\rm MRMLKSVM_{RBF}$, $\rm DuSK_{RBF}$ and $\rm SVM_{RBF}$ respectively while we set $K(\textbf{X},\textbf{Y})=\textbf{X}^{\intercal}\textbf{Y}+\sigma$ in KSMM.

\begin{table}
\caption{Prediction performance in simulation study in terms of accuracy }
\begin{tabular*}{1\textwidth}{@{\extracolsep{\fill}}  llllllll}
\noalign{\smallskip}
\hline
\noalign{\smallskip}
N \ \ \ & p  \ \ \ & \multicolumn{5}{l}{Accuracy $(\%)$} \\
& & STM  \ \ \ \ \ \ \ & $\rm SVM_{RBF}$ \ \ \ \ & $\rm DuSK_{RBF}$ \ \ \ & $\rm MRMLKSVM_{RBF}$ & SMM & $\rm KSMM_{Linear}$  \\
\noalign{\smallskip}
\hline
\noalign{\smallskip}
50 & 10 & 76.4(6.7) & 80.8(6.3) & 81.2(5.2) & 80.4(5.0) & 81.2(5.1) & \textbf{83.2(6.1)} \\
& 15 & 82.8(8.9) & 89.2(4.1) & 90.0(1.8) & 90.4(3.4) & 87.9(3.5) & \textbf{92.4(2.7)} \\
& 20 & 85.6(5.9) & 89.6(3.4) & 88.8(6.4) & 88.8(4.3) & 88.6(2.5) & \textbf{91.2(2.4)} \\
& 25 & 88.0(4.2) & 92.8(1.0) & 91.6(2.0) & 93.6(0.8) & 92.2(3.2) & \textbf{94.0(1.3)} \\
& 30 & 84.8(2.0) & 89.6(3.4) & 92.4(2.9) & 93.2(3.2) & 92.8(2.2) & \textbf{95.2(2.0)} \\
\noalign{\smallskip}
\hline
\noalign{\smallskip}
100 & 10 & 81.8(4.9) & 86.0(3.2) & 85.0(4.1) & 79.4(7.3) & 81.4(3.2) & \textbf{86.6(2.9)} \\
& 15 & 85.2(3.2) & 89.0(3.2) & 87.8(2.1) & 87.6(3.0) & 86.7(1.2) & \textbf{89.2(2.1)} \\
& 20 & 84.8(6.7) & 90.2(2.1) & 90.0(2.1) & 91.0(2.0) & 89.5(1.4) & \textbf{91.2(2.5)} \\
& 25 & 89.4(3.2) & 92.0(1.1) & 91.6(2.3) & 93.2(1.5) & 93.1(2.1) & \textbf{93.4(2.2)} \\
& 30 & 86.8(7.4) & 93.6(3.4) & 92.2(2.1) & \textbf{94.8(2.9)} & 93.2(2.9) & \textbf{94.8(2.8)} \\
\noalign{\smallskip}
\hline
\label{Table1}
\end{tabular*}
\end{table}


The results are presented in Table \ref{Table1}. We can observe that KSMM performs well in general. We are interested in accuracy in comparison and one way to understand this is to realize that our kernels are represented as matrices in calculation and Newton's method is included which occupies much more space and time. In addition, the observations demonstrate the size of training set has positive effect on the performance in most cases. KSMM has a significant performance even the sample size is small. When the training set is large enough, the accuracy is increasing along with the growing number of attributes. That is reasonable for the expectation values of examples in two classes are equal to $p\textbf{A}$ and $2p\textbf{A}$  respectively which make it easier to identify as $p$ increases. 

\subsection{Datasets and Discussion}
Next, we evaluate the performance of our classifier on real data sets coming from variety of domains. We consider the following benchmark datasets to perform a series of comparative experiments on multiple classification problems.
We use the ORL$32\times 32$\footnote{\url{http://www.cl.cam.ac.uk/research/dtg/attarchive/facedatabase.html}}  \citep{SamariaH94}, the Sheffield Face dataset\footnote{\url{https://www.sheffield.ac.uk/eee/research/iel/research/face}}, the Columbia Object Image Library (COIL-20)\footnote{\url{http://www.cs.columbia.edu/CAVE/software/softlib/coil-20.php}} \citep{Nene96columbiaobject} and the Binary Alphadigits\footnote{\url{http://www.cs.toronto.edu/~roweis/data.html}}. To better visualize the experimental data, we randomly choose a small subset for each database, as shown in Fig.~\ref{Fig2}. Table  \ref{Table2} summarizes the properties of all datasets used in our experiments.

\begin{table}
\caption{Statistics of datasets used in our experiments.}
\begin{tabular*}{1\textwidth}{@{\extracolsep{\fill}}  llll}
\noalign{\smallskip}
\hline
\noalign{\smallskip}
Dataset & \#Instances & \#Class & Size \\
\noalign{\smallskip}
\hline
\noalign{\smallskip}
ORL$32\times32$ & 400 & 40 & $32\times32$ \\
UMIST & 564 & 20 & $112\times92$ \\
COIL-20 & 1440 & 20 & $128\times 128$ \\
Binary Alphadigits & 1404 & 36 & $20\times16$ \\
\noalign{\smallskip}
\hline
\label{Table2}
\end{tabular*}
\end{table}

\begin{figure}
\centering
\subfigure[]{
\includegraphics[width=0.4\textwidth, height=42mm]{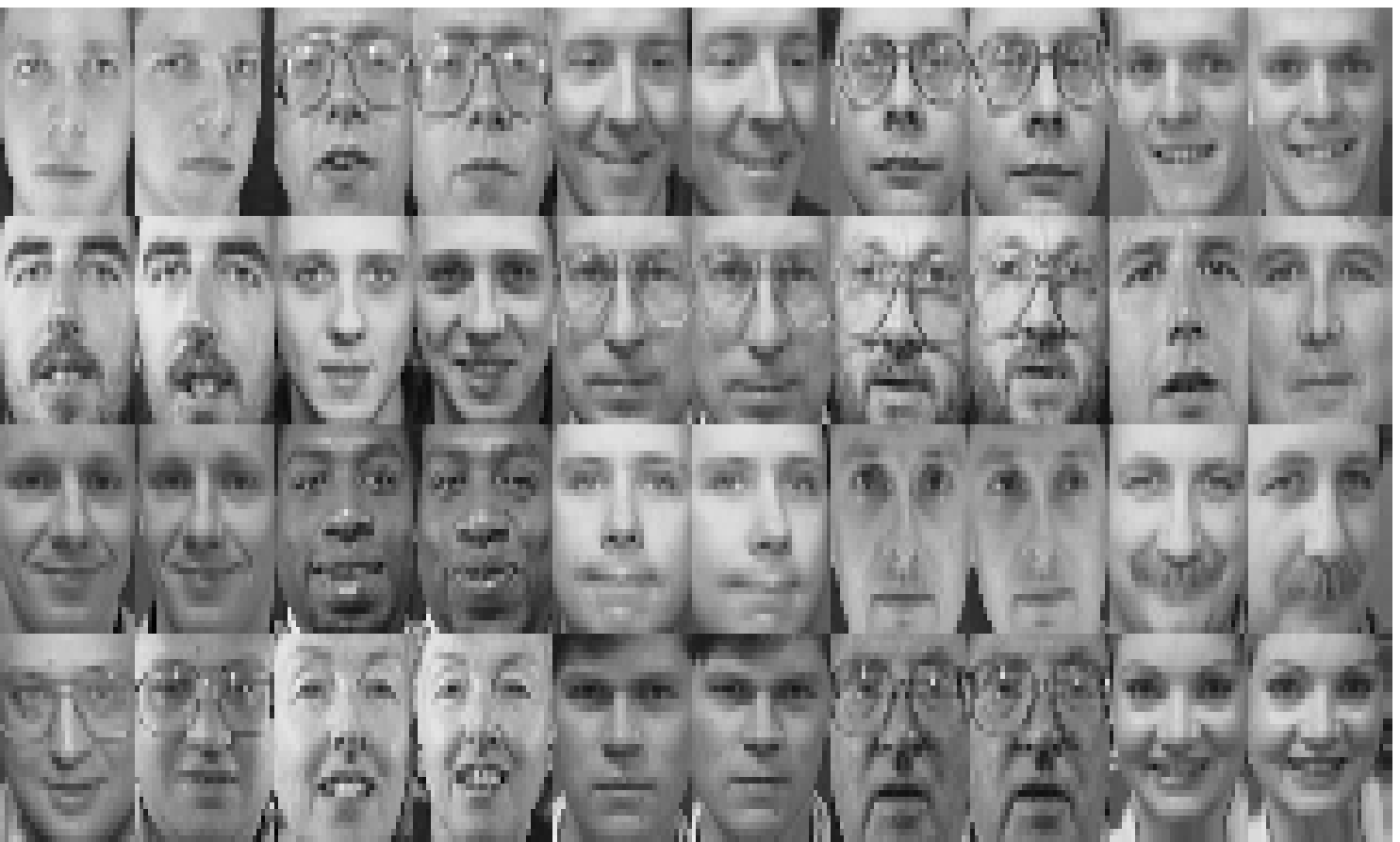}}
\subfigure[]{
\includegraphics[width=0.4\textwidth, height=42mm]{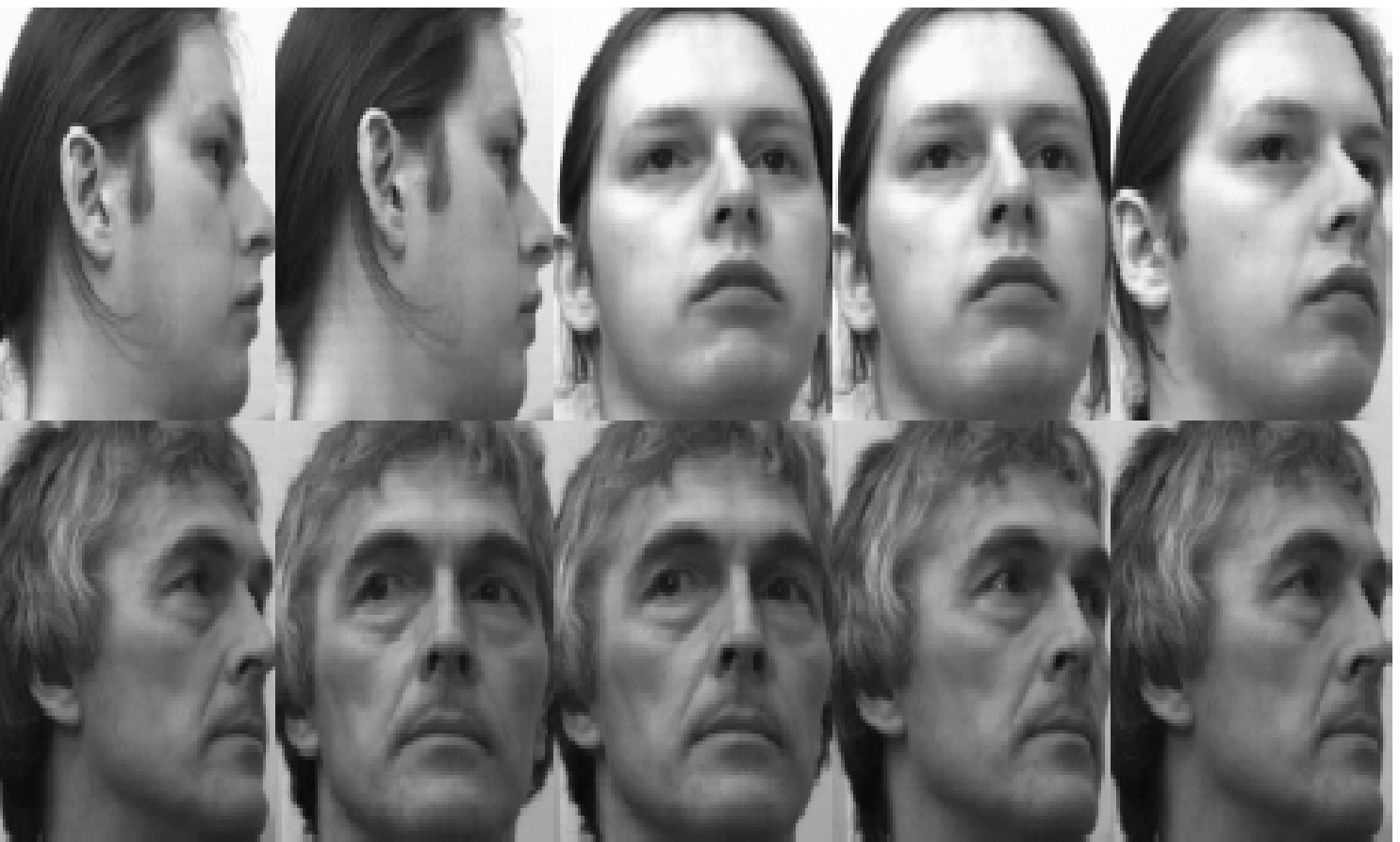}}
\subfigure[]{
\includegraphics[width=0.4\textwidth, height=42mm]{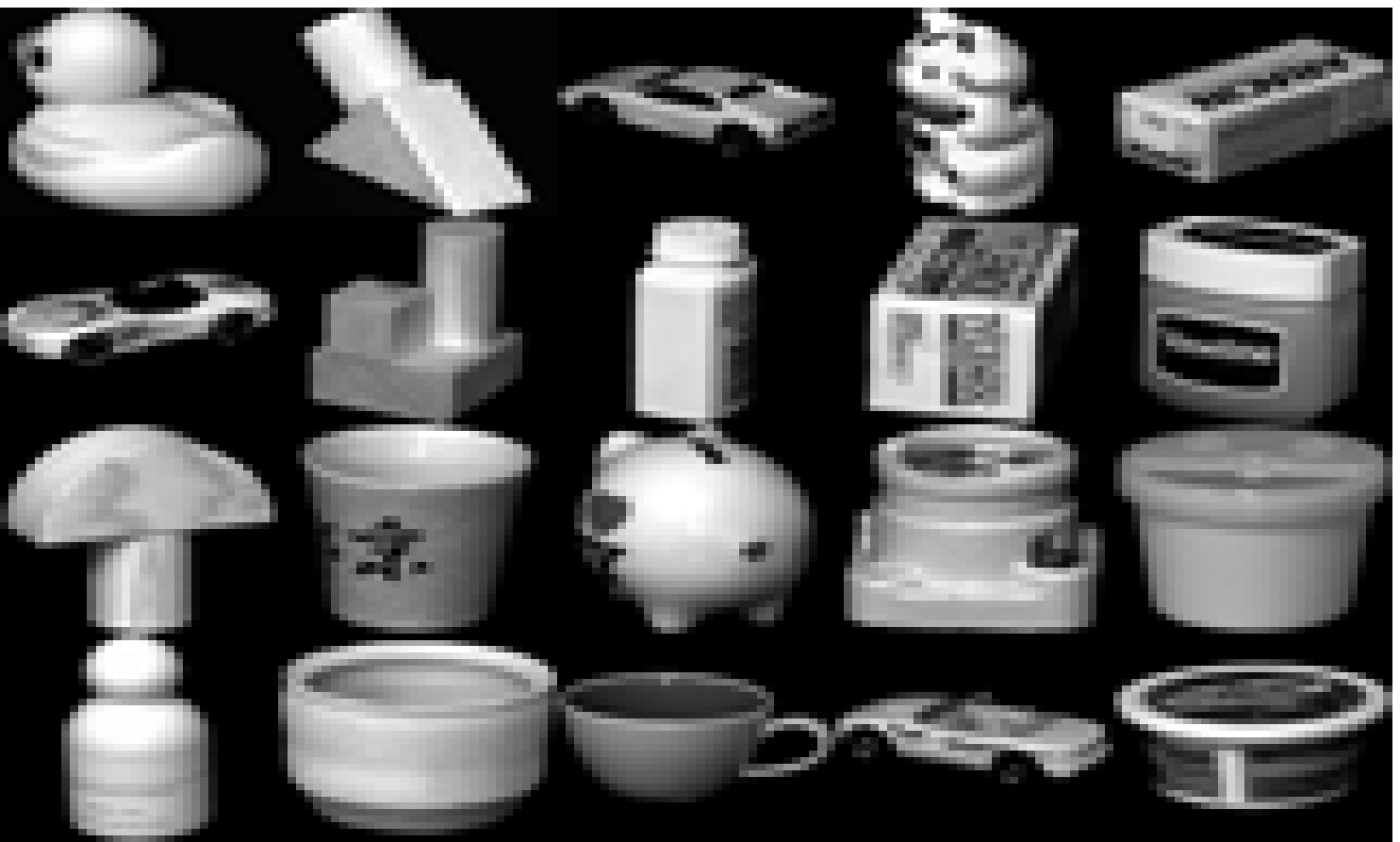}}
\subfigure[]{
\includegraphics[width=0.4\textwidth, height=42mm]{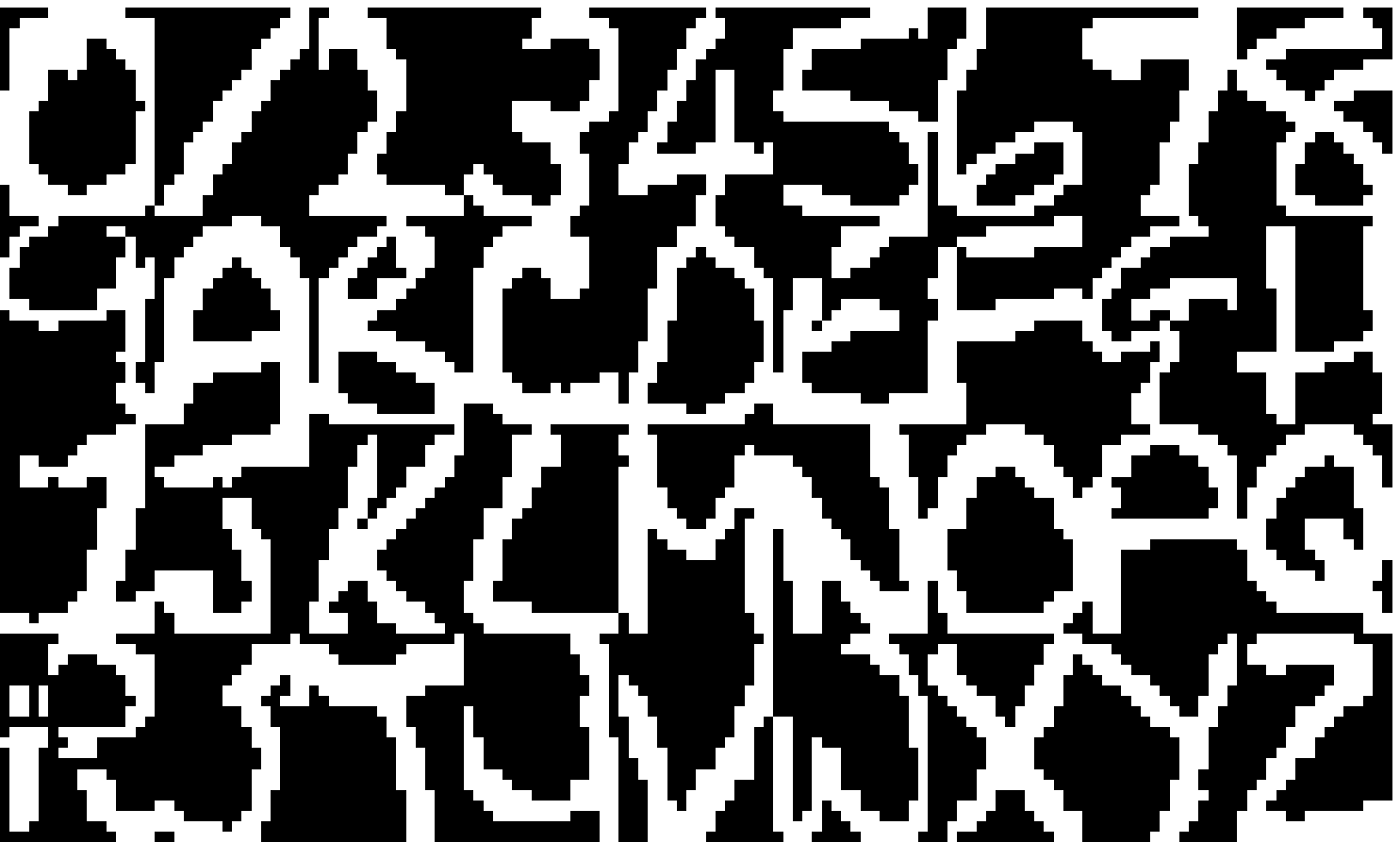}}
\caption{Example images for classification problems. \textbf{a} ORL$32\times 32$,  \textbf{b} Sheffield Face dataset, \textbf{c} Binary Alphadigits, \textbf{d} COIL-20}
\label{Fig2}
\end{figure}

The ORL$32\times 32$ contains 40 distinct subjects of each of ten different images with $32 \times 32$ pixels. For some subjects, the images were taken at different times, varying the lighting, facial expressions and facial details.
The Sheffield (previously UMIST) Face Database consists of 564 images of 20 individuals. Each individual is shown in a range of poses from profile to frontal views at the $112 \times 92$ field and images are numbered consecutively as they were taken.
The COIL-20 is a database of two sets of images in which 20 objects were placed on a motorized turntable against background. We use the second one of 1440 images with backgrounds discarded and sizes normalized, each of which has $128 \times 128$ pixels. We crop all images into $32 \times 32$ pixels to efficiently apply above algorithms.
The Binary Alphadigits is composed of digits of ``0'' to ``9'' and capital ``A'' to ``Z'' with $20\times16$ pixels, each of which has 39 examples.
In experiments, we randomly choose $50\%$ of images of each individual together as the training set and other images retained as test set for multiple classification.

Note that parameters of different algorithms are set as in the simulation study. For each setting we average results over 10 trials each of which are obtained from randomly divide each dataset into two subsets, one for training and one for testing. For multiple classification task, we use the strategy of one-against-one (1-vs-1) method. For KSMM, we set $K(\textbf{X},\textbf{Y})=[\exp(-\sigma \|\textbf{X}(:,i)-\textbf{Y}(:,j)\|^2)]_{n \times n}$ as the matrix kernel function. Due to the effectiveness of SMM in dealing with data matrices, we examine the convergence behavior in terms of the number of iterations of SMM and KSMM. 

\begin{figure}
\centering
\subfigure[]{
\includegraphics[width=0.8\textwidth]{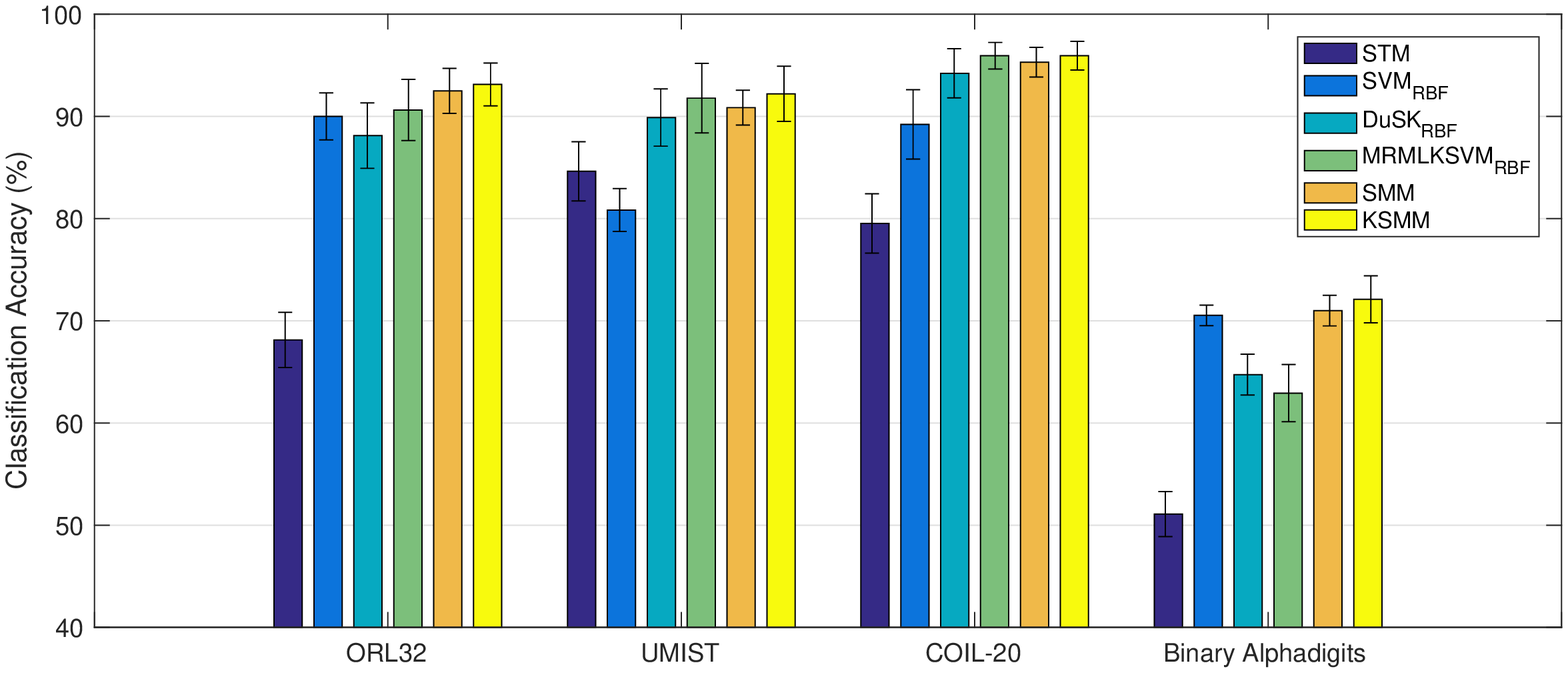}}
\subfigure[]{
\includegraphics[width=0.8\textwidth]{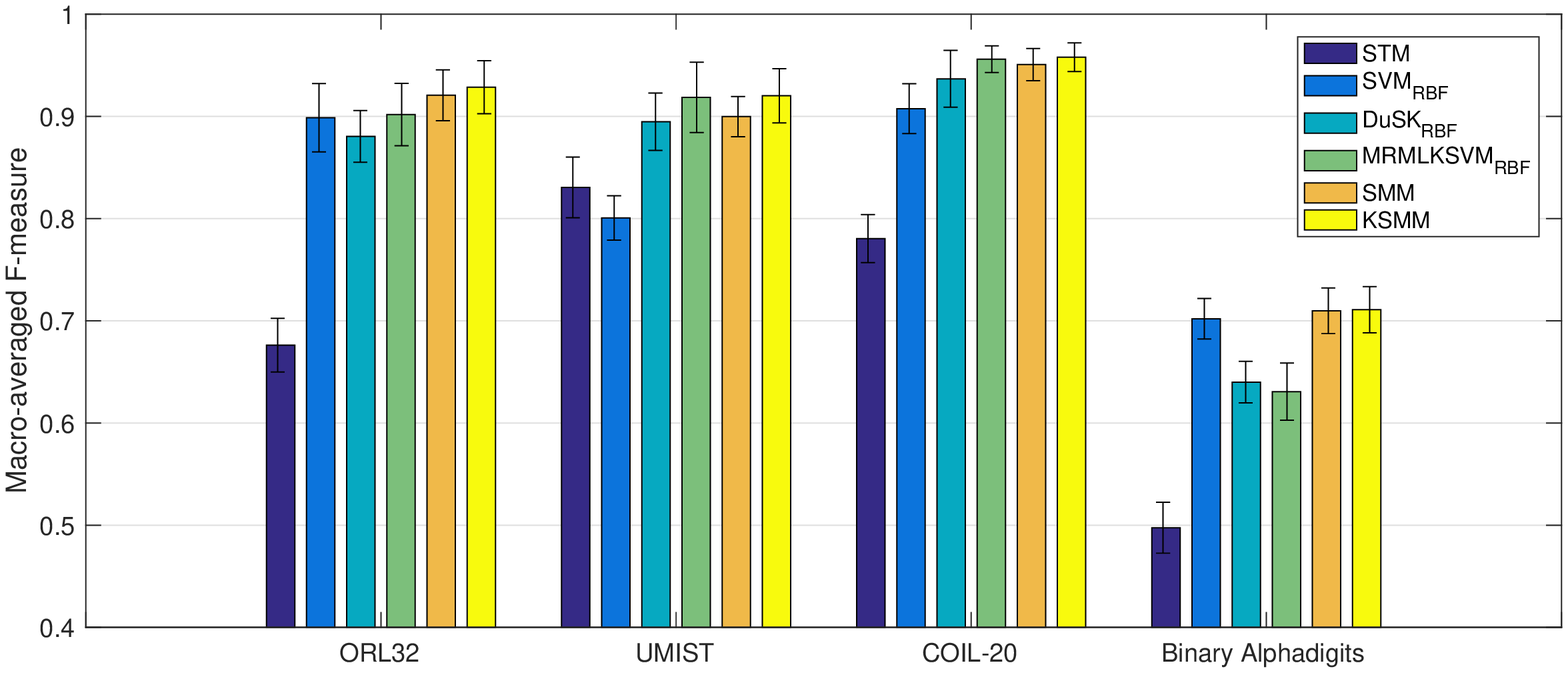}}
\caption{Accuracy and macro-averaged F-measure on benchmark datasets. We plot avg.accuracy($\%$) and F-measure $\pm$ standard error for certain classifiers. \textbf{a} Accuracy,  \textbf{b} F-meature}
\label{Fig3}
\end{figure}

Fig.~\ref{Fig3} and Fig.~\ref{Fig4} summarize experimental results for above datasets. Similar patterns of learning curves are observed in macro-averaged F-measure and accuracy, which shows that KSMM outperforms the baseline methods. We can see that KSMM obtains a better result though SMM exhibits a faster convergence than KSMM, which means that KSMM occupies more time. The SVM approach gives slightly worse result on UMIST, for structural information is broken by straightly convert matrices into vectors. It is worth noting that on Binary Alphadigits dataset it is very hard for classification algorithms to achieve satisfying accuracy since the dimension is low and some labels are rather difficult to identify, e.g. digit ``0'' and letter ``O'', digit ``1'' and letter ``I''. These results clearly show that KSMM can successfully deal with classification problems. One explanation of the outstanding performance of our method is due to each entry of the matrix inner product $\langle \textbf{W},\textbf{X}\rangle_{\mathcal{H}}$ measures a ``distance'' from \textbf{X} to a certain hyperplane. The final strategy focuses on the weighted summation of these values with weight matrix \textbf{V}. However, most methods in the literature tries to separate two classes upon one single hyperplane, even applying the magic of kernels to transform a nonlinear separable problem into a linear separable one in rather high dimension.

\begin{figure}
\centering
\subfigure[ORL32]{
\includegraphics[width=0.48\textwidth]{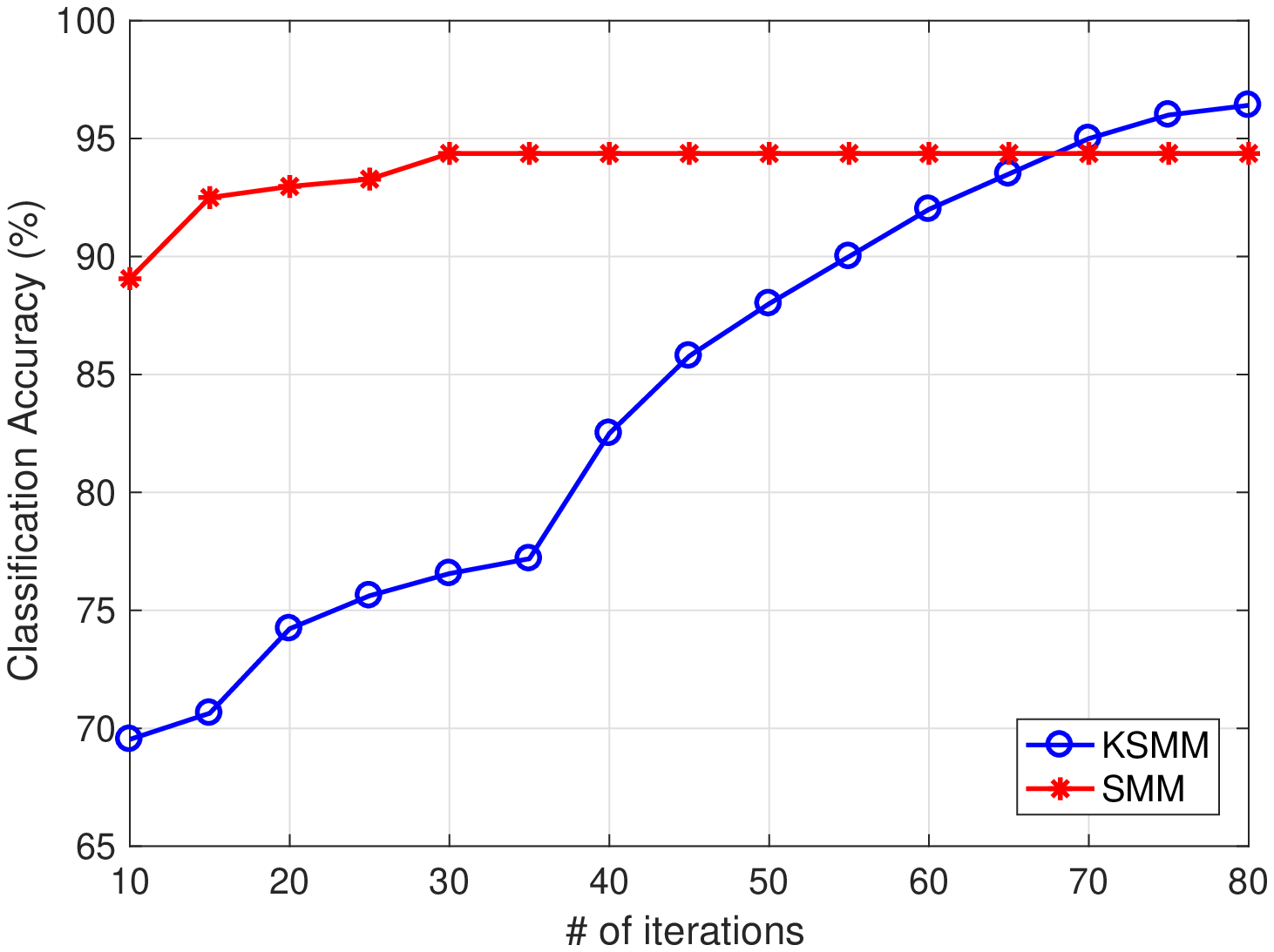}}
\subfigure[UMIST]{
\includegraphics[width=0.48\textwidth]{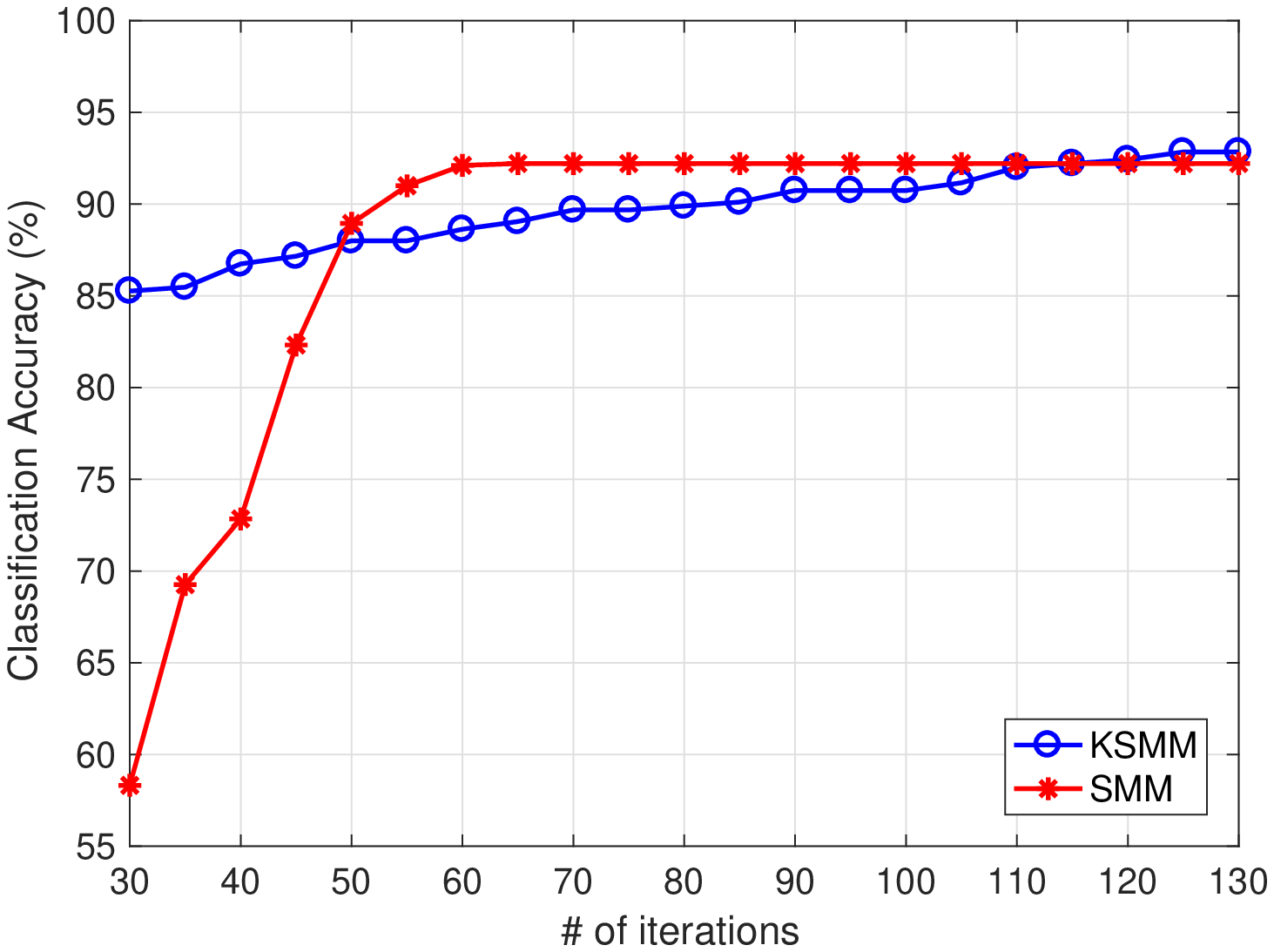}}
\subfigure[COIL-20]{
\includegraphics[width=0.48\textwidth]{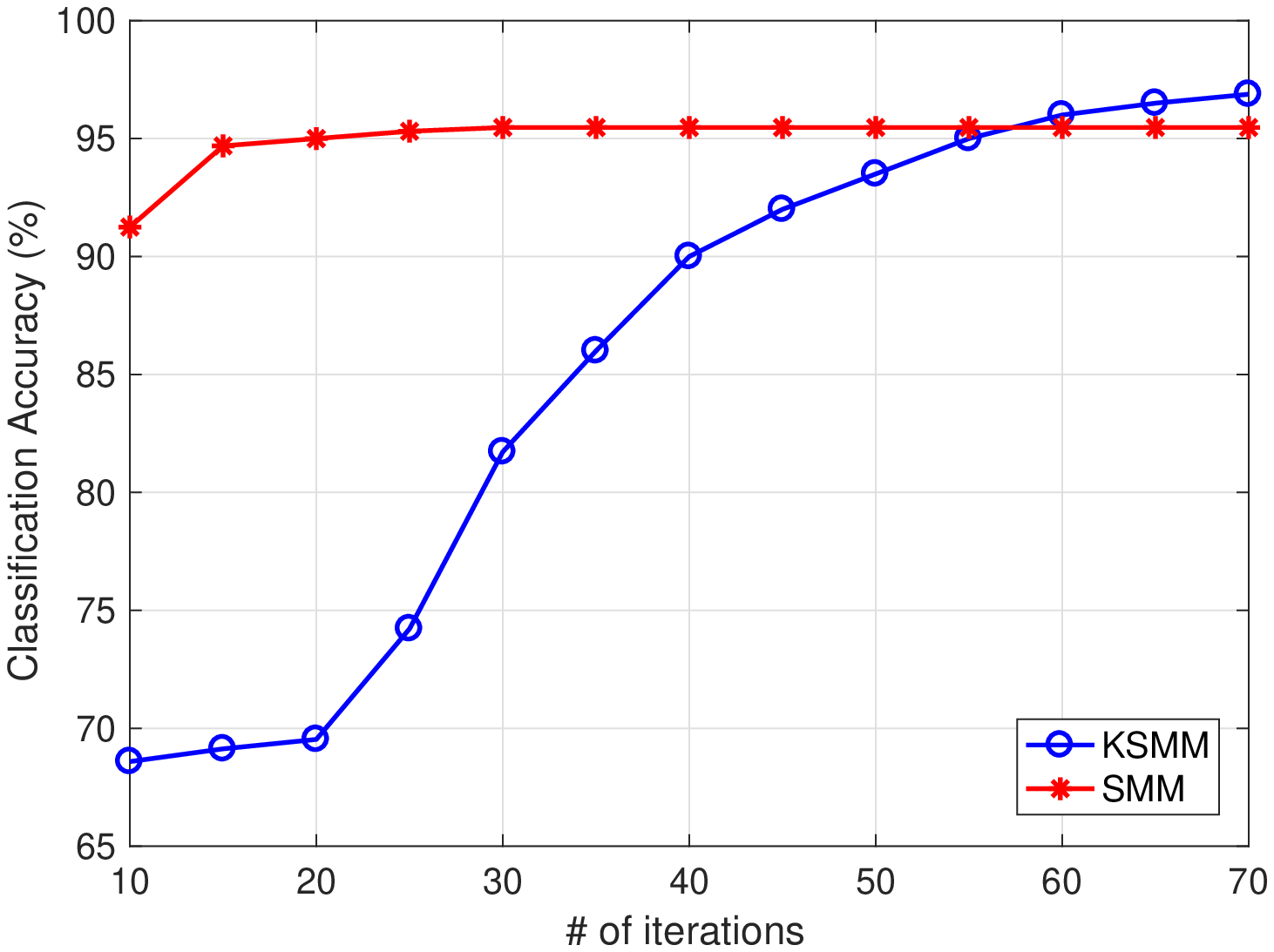}}
\subfigure[Binary Alphadigits]{
\includegraphics[width=0.48\textwidth]{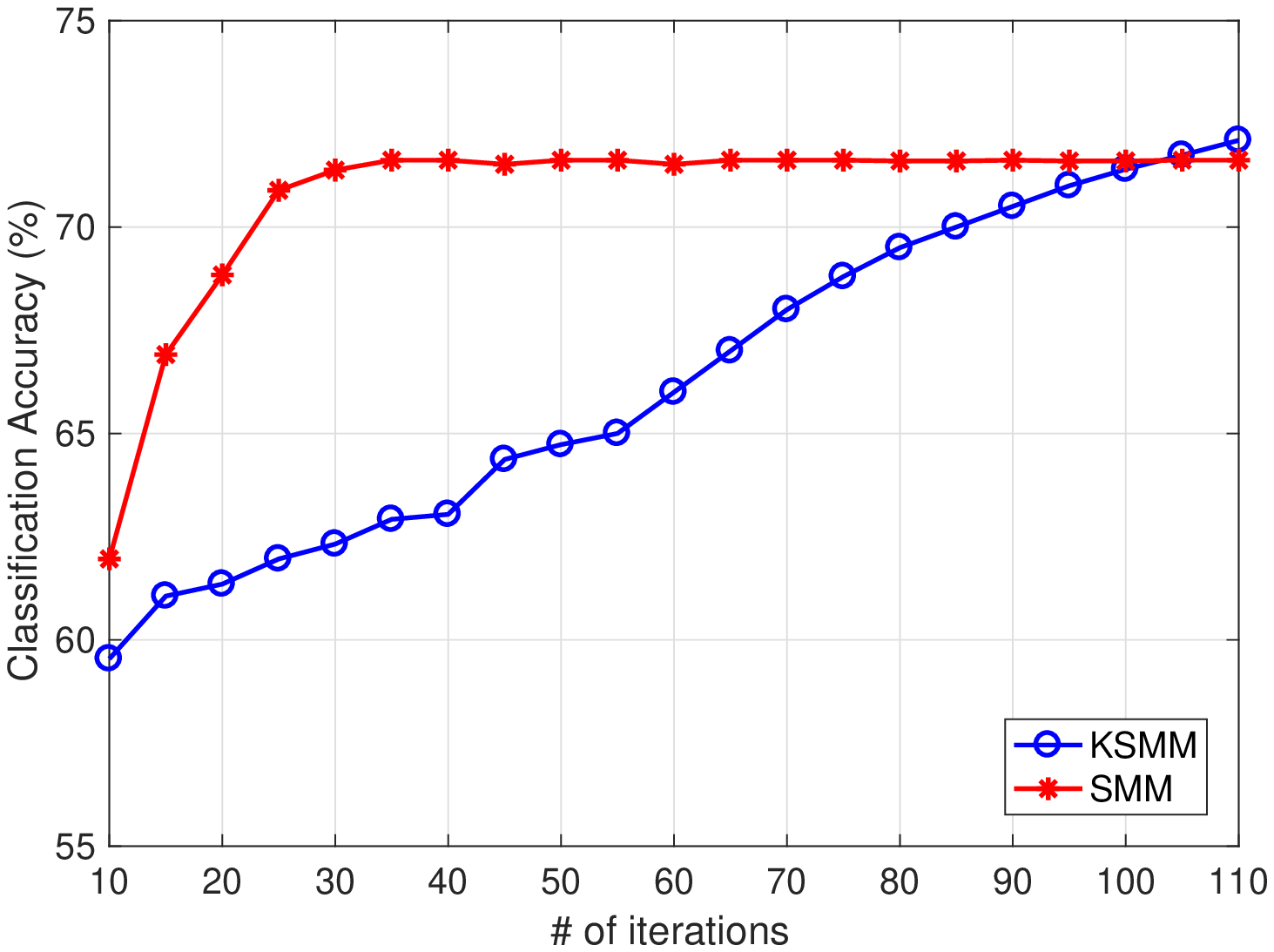}}
\caption{Comparing the accuracy versus the number of iterations of SMM and KSMM for solving different tasks. }
\label{Fig4}
\end{figure}

Overall, the results indicate that KSMM is a significantly effective and competitive alternative for both binary and multiple classification. Note that any reasonable matrix kernel function can be applied in this study.

\section{Concluding Remarks}\label{cr}
Kernel support matrix machine provides a principled way of separating different classes via their projections in a Reproducing Kernel Matrix Hilbert Space. In this paper, we have showed how to use matrix kernel functions to discover the structural similarities within classes for the construction of proposed hyperplane. The theoretical analysis of its generalization bounds highlights the reliability and robustness of KSMM in practice. Intuitively, the optimization problem arising in KSMM only needs to be solved once while other tensor-based classifiers, such as  STM, MRMLKSVM need to be solved iteratively.

As our experimental results demonstrate, KSMM is competitive in terms of accuracy with state-of-the-art classifiers on several classification benchmark datasets. As previous work focuses on decomposing original data as sum of low rank factors, this paper provides a new insight into exploiting the structural information of matrix data.

In future work, we will seek technical solutions of (\ref{dp}) to improve efficiency or figure out other approach to the use of matrix Hilbert space since the problem we analyze here is non-convex. We could only obtain a local optimal solution other than a global one which might deteriorate the performance of KSMM in experiments. Another interesting topic would be to design specialized method to learn the matrix kernel and address parameters. Figuring out that matrix kernel functions and supervised tensor learning are closely related, hence, a natural extension to this work is the derivation of a unifying matrix kernel-based framework for regression, clustering, among other tasks.

\begin{acknowledgements}
The work is supported by National Natural Science Foundations of China under Grant 11531001 and National Program on Key Basic Research Project under Grant 2015CB856004. We are grateful to Dong Han for our discussions.
\end{acknowledgements}

\section*{Appendix}
\begin{appendix}

\section{Proof of Theorem \ref{ksmmb}
}\label{Aa}
First, we recall some basic notations that are useful to our analysis.

The Rademacher complexity of $\mathcal{F}$ with respect to $\mathcal{S}$ is defined as follows:
\begin{equation*}
R(\mathcal{F} \circ \mathcal{S}) = \frac{1}{N} \underset{\bm{\sigma} \sim \{\pm1\}^N}{\mathbb{E}}\bigg[\sup\limits_{f \in \mathcal{F}} \sum_{i=1}^N \sigma_i f(z_i)\bigg].
\end{equation*}
More generally, given a set of vectors, $\mathcal{A} \subset \mathbb{R}^N$, we define
\begin{equation*}
R(\mathcal{A}) = \frac{1}{N} \underset{\bm{\sigma}}{\mathbb{E}}\bigg[\sup\limits_{\textbf{a} \in \mathcal{A}} \sum_{i=1}^N \sigma_i a_i\bigg].
\end{equation*}

In order to prove the theorem we rely on the generalization bounds for KSMM, we show the following lemmas to support our conclusion.

\begin{lemma}\label{l1}
Assume that for all z and $h \in \mathcal{H}_p$ we have that $|l(h,z)| \leq c$, then with probability at least $1-\delta$, for all $h \in \mathcal{H}_p$,
\begin{equation}
L_D(h)-L_{\mathcal{S}}(h) \leq 2 \underset{\mathcal{S}' \sim D^N}{\mathbb{E}} R(\ell \circ \mathcal{H}_p \circ \mathcal{S}')+c\sqrt\frac{2\ln(2/\delta)}{N}.
\end{equation}
\end{lemma}

\begin{lemma}\label{l2}
For each $i=1,\cdots,N$, let $\Phi_i : \mathbb{R} \rightarrow \mathbb{R}$ be a $\rho$-Lipschitz function, namely for all $\alpha,\beta \in \mathbb{R}$ we have $|\Phi_i(\alpha)-\Phi_i(\beta)| \leq \rho |\alpha-\beta|$. For $\textbf{a} \in \mathbb{R}^N$, let $\Phi(\textbf{a})$ denote the vector $(\Phi_1(a_1), \cdots, \Phi_N(a_N))$ and $\Phi \circ \mathcal{A} = \{\Phi(\textbf{a}): \textbf{a} \in \mathcal{A}\}$. Then,
\begin{equation}
R(\Phi \circ \mathcal{A}) \leq \rho R(\mathcal{A}).
\end{equation}
\end{lemma}
The proof of Lemma \ref{l1} and \ref{l2} can be discovered in \citep{Shalev2014Understanding}. Additionally, we present the next lemma.

\begin{lemma}\label{l3}
Let $\mathcal{S}=(\emph{\textbf{X}}_1,\cdots,\emph{\textbf{X}}_N)$ be a finite set of matrices in a matrix Hilbert space $\mathcal{H}$. Define $\mathcal{H} \circ \mathcal{S} = \{(\langle \langle \emph{\textbf{W}},\emph{\textbf{X}}_1 \rangle_{\mathcal{H}},\frac{\emph{\textbf{V}}}{\|\emph{\textbf{V}}\|}\rangle,\cdots,\langle \langle \emph{\textbf{W}},\emph{\textbf{X}}_N \rangle_{\mathcal{H}},\frac{\emph{\textbf{V}}}{\|\emph{\textbf{V}}\|}\rangle):\|\emph{\textbf{W}}\|_{\mathcal{H}}(\emph{\textbf{V}}) \leq 1\}$. Then,
\begin{equation}
R(\mathcal{H} \circ \mathcal{S}) \leq \frac{\max_i \|\emph{\textbf{X}}_i\|_{\mathcal{H}}(\emph{\textbf{V}})}{\sqrt{N}}.
\end{equation}
\end{lemma}
\begin{proof}
Using Cauchy-Schwartz inequality, we derive the following inequality
\begin{equation}
\begin{split}
N R(\mathcal{H} \circ \mathcal{S}) &= \underset{\bm{\sigma}}{\mathbb{E}} \bigg[\sup\limits_{\textbf{a} \in \mathcal{H} \circ \mathcal{S}} \sum_{i=1}^N \sigma_i a_i\bigg] \\
&= \underset{\bm{\sigma}}{\mathbb{E}} \bigg[\sup\limits_{\textbf{W}:\|\textbf{W}\|_{\mathcal{H}}(\textbf{V}) \leq 1} \sum_{i=1}^N \sigma_i \langle \langle \textbf{W},\textbf{X}_i \rangle_{\mathcal{H}},\frac{\textbf{V}}{\|\textbf{V}\|}\rangle\bigg] \\
&= \underset{\bm{\sigma}}{\mathbb{E}} \bigg[\sup\limits_{\textbf{W}:\|\textbf{W}\|_{\mathcal{H}}(\textbf{V}) \leq 1} \langle \langle \textbf{W}, \sum_{i=1}^N \sigma_i \textbf{X}_i \rangle_{\mathcal{H}},\frac{\textbf{V}}{\|\textbf{V}\|}\rangle\bigg] \\
& \leq \underset{\bm{\sigma}}{\mathbb{E}} \bigg[ \| \sum_{i=1}^N\sigma_i \textbf{X}_i \|_{\mathcal{H}}(\textbf{V}) \bigg].
\end{split}
\end{equation}

Next, using Jensen's inequality we have that
\begin{equation}
\underset{\bm{\sigma}}{\mathbb{E}} \bigg[ \| \sum_{i=1}^N\sigma_i \textbf{X}_i \|_{\mathcal{H}}(\textbf{V}) \bigg]=\underset{\bm{\sigma}}{\mathbb{E}} \bigg[ \Big(\| \sum_{i=1}^N\sigma_i \textbf{X}_i \|_{\mathcal{H}}^2(\textbf{V}) \Big)^{1/2} \bigg] \leq \Big(\underset{\bm{\sigma}}{\mathbb{E}} \bigg[ \| \sum_{i=1}^N\sigma_i \textbf{X}_i \|_{\mathcal{H}}^2(\textbf{V}) \bigg]\Big)^{1/2}.
\end{equation}

Since the variables $\sigma_1,\cdots,\sigma_N$ are independent we have
\begin{equation*}
\begin{split}
\underset{\bm{\sigma}}{\mathbb{E}} \bigg[ \| \sum_{i=1}^N\sigma_i \textbf{X}_i \|_{\mathcal{H}}^2(\textbf{V}) \bigg] &= \underset{\bm{\sigma}}{\mathbb{E}} \bigg[  \sum_{i,j=1}^N \sigma_i \sigma_j \langle \langle \textbf{X}_i,\textbf{X}_j \rangle_{\mathcal{H}},\frac{\textbf{V}}{\|\textbf{V}\|}\rangle \bigg] \\
&=\sum_{i \neq j} \langle \langle \textbf{X}_i,\textbf{X}_j \rangle_{\mathcal{H}},\frac{\textbf{V}}{\|\textbf{V}\|}\rangle \underset{\bm{\sigma}}{\mathbb{E}} [\sigma_i \sigma_j] + \sum_{i=1}^N \langle \langle \textbf{X}_i,\textbf{X}_i \rangle_{\mathcal{H}},\frac{\textbf{V}}{\|\textbf{V}\|}\rangle \underset{\bm{\sigma}}{\mathbb{E}} [\sigma_i^2] \\
&= \sum_{i=1}^N \|\textbf{X}_i\|_{\mathcal{H}}^2(\textbf{V}) \leq N \max_i \|\textbf{X}_i\|_{\mathcal{H}}^2(\textbf{V}).
\end{split}
\end{equation*}

Combining these inequalities we conclude our proof.
\end{proof}
\qed

Finally, we complete our proof as follows. Let $\mathcal{F} = \{ (\textbf{X},y) \mapsto \Phi(\langle \langle \textbf{W}',\textbf{X} \rangle_{\mathcal{H}},\frac{\textbf{V}}{\|\textbf{V}\|}\rangle,y) : \textbf{W}' \in \mathcal{H}'_p\}$. Indeed, the set $\mathcal{F} \circ \mathcal{S}$ can be written as
\begin{equation*}
\mathcal{F} \circ \mathcal{S}=\{(\Phi(\langle \langle \textbf{W}',\textbf{X}_1 \rangle_{\mathcal{H}},\frac{\textbf{V}}{\|\textbf{V}\|}\rangle,y_1),\cdots,\Phi(\langle \langle \textbf{W}',\textbf{X}_N \rangle_{\mathcal{H}},\frac{\textbf{V}}{\|\textbf{V}\|}\rangle,y_N)) : \textbf{W}' \in \mathcal{H}'_p\},
\end{equation*}
and $R(\mathcal{F} \circ \mathcal{S}) \leq \frac{\rho B' R'}{\sqrt{N}}$ with probability 1 follows directly by combining Lemma \ref{l2} and \ref{l3}. Then the claim of Theorem \ref{ksmmb} follows from Lemma \ref{l1}.

\section{Proof of Theorem \ref{ksmmb2}
}\label{Ab}
First, we summarize the following lemma \citep{Shalev2014Understanding}, due to Massart, which states that the Rademacher complexity of a finite set grows logarithmically with the size of the set.
\begin{lemma}[Massart lemma]\label{l4}
Let $\mathcal{A}=\{\textbf{a}_1,\cdots,\textbf{a}_N\}$ be a finite set of vectors in $\mathbb{R}^m$. Define $\bar{\textbf{a}}=\frac{1}{N} \sum_{i=1}^N \textbf{a}_i$. Then,
\begin{equation}
R(\mathcal{A}) \leq \max_{\textbf{a} \in \mathcal{A}} \|\textbf{a}-\bar{\textbf{a}}\|_2 \frac{\sqrt{2\ln{N}}}{m}.
\end{equation}
\end{lemma}

Define $\mathcal{H}_{\max} \circ \mathcal{S} = \{(\langle \langle \textbf{W},\textbf{X}_1 \rangle_{\mathcal{H}},\frac{\textbf{V}}{\|\textbf{V}\|}\rangle,\cdots,\langle \langle \textbf{W},\textbf{X}_N \rangle_{\mathcal{H}},\frac{\textbf{V}}{\|\textbf{V}\|}\rangle):\|\textbf{W}\|_{\max} \leq 1\}$. Next we bound the Rademacher complexity of $\mathcal{H}_{\max} \circ \mathcal{S}$.

\begin{lemma}\label{l5}
Let $\mathcal{S}=\{\emph{\textbf{X}}_1,\cdots,\emph{\textbf{X}}_N\}$ be a finite set of matrices in $\mathbb{R}^{m \times n}$. Then,
\begin{equation}
R(\mathcal{H}_{\max} \circ \mathcal{S}) \leq n \max_{1 \leq i \leq N} \| \emph{\textbf{X}}_i\|_1 \sqrt{\frac{2(m \ln 2+\ln n)}{N}}.
\end{equation}
\end{lemma}
\begin{proof}
Using inequality (\ref{HleF}), we have
\begin{equation}\label{max1}
\begin{split}
N R(\mathcal{H}_{\max} \circ \mathcal{S}) &= \underset{\bm{\sigma}}{\mathbb{E}} \bigg[\sup\limits_{\textbf{a} \in \mathcal{H}_{\max} \circ \mathcal{S}} \sum_{i=1}^N \sigma_i a_i\bigg] \\
&= \underset{\bm{\sigma}}{\mathbb{E}} \bigg[\sup\limits_{\textbf{W}:\|\textbf{W}\|_{\max} \leq 1} \sum_{i=1}^N \sigma_i \langle \langle \textbf{W},\textbf{X}_i \rangle_{\mathcal{H}},\frac{\textbf{V}}{\|\textbf{V}\|}\rangle\bigg] \\
&= \underset{\bm{\sigma}}{\mathbb{E}} \bigg[\sup\limits_{\textbf{W}:\|\textbf{W}\|_{\max} \leq 1} \langle \langle \textbf{W}, \sum_{i=1}^N \sigma_i \textbf{X}_i \rangle_{\mathcal{H}},\frac{\textbf{V}}{\|\textbf{V}\|}\rangle\bigg] \\
& \leq \underset{\bm{\sigma}}{\mathbb{E}} \bigg[\sup\limits_{\textbf{W}:\|\textbf{W}\|_{\max} \leq 1} \| \langle \textbf{W}, \sum_{i=1}^N \sigma_i \textbf{X}_i \rangle_{\mathcal{H}} \| \bigg] \\
& = \underset{\bm{\sigma}}{\mathbb{E}} \bigg[\sup\limits_{\textbf{W}:\|\textbf{W}\|_{\max} \leq 1} \| \textbf{W}^{\intercal} \sum_{i=1}^N \sigma_i \textbf{X}_i \| \bigg] \\
& \leq  \underset{\bm{\sigma}}{\mathbb{E}} \bigg[\sup\limits_{\textbf{W}:\|\textbf{W}\|_{\max} \leq 1} n \| \textbf{W}\|_{\max} \| \sum_{i=1}^N \sigma_i \textbf{X}_i \|_1 \bigg] \\
& = n  \underset{\bm{\sigma}}{\mathbb{E}} \bigg[\| \sum_{i=1}^N \sigma_i  \textbf{X}_i \|_1 \bigg].
\end{split}
\end{equation}
For each $j=1,\cdots,n$, we define $\textbf{u}_j^{\bm{\gamma} \in \{\pm 1\}^m}=(\sum\limits_{i=1}^m \gamma_i [\textbf{X}_1]_{i j},\cdots, \sum\limits_{i=1}^m \gamma_i [\textbf{X}_N]_{i j}) \in \mathbb{R}^N$. Note that $\|\textbf{u}_j^{\bm{\gamma}} \|_2 \leq \sqrt{N} \max_{1 \leq i \leq N} \|\textbf{X}_i\|_1$. Let $\mathcal{U}=\{\textbf{u}_j^{\bm{\gamma}} : j=1,\cdots,n, \bm{\gamma} \in \{\pm 1\}^m\}$. The right-hand side of Equation (\ref{max1}) is $NnR(\mathcal{U})$. Using Massart lemma (Lemma \ref{l4}) we have that
\begin{equation*}
R(\mathcal{U})  \leq \max_{1 \leq i \leq N} \|\textbf{X}_i\|_1 \sqrt{\frac{2(m \ln 2+\ln n)}{N}},
\end{equation*}
which concludes our proof.
\end{proof}
\qed

The rest of the proof is identical to the proof of Theorem \ref{ksmmb}, while relying on Lemma \ref{l5} instead of relying on Lemma \ref{l3}.
\end{appendix}

\bibliographystyle{spbasic}      
\bibliography{Reference}   

%
%

\end{document}